\documentclass{article}

\usepackage[final, nonatbib]{neurips_2020}

\usepackage[utf8]{inputenc} 
\usepackage[T1]{fontenc}    
\usepackage{hyperref}       
\usepackage{url}            
\usepackage{booktabs}       
\usepackage{amsfonts}       
\usepackage{nicefrac}       
\usepackage{microtype}      

\usepackage{graphicx}
\usepackage{subfigure}
\usepackage{amsmath}

\usepackage{bm}
\usepackage{booktabs} 
\usepackage{amsmath}
\usepackage{algorithmic}
\usepackage{algorithm}
\usepackage{multirow}
\usepackage{graphicx}
\usepackage{subfigure}
\usepackage{wrapfig}
\usepackage{amsthm}
\usepackage{amsfonts}
\usepackage{graphicx}
\usepackage{wrapfig}
\usepackage{caption}

\newtheorem{definition}{Definition} 
\newtheorem*{definition_1}{Definition 1} 
\newtheorem{assumption}{Assumption}
\newtheorem*{assumption_1}{Assumption 1}
\newtheorem{proposition}{Proposition}
\newtheorem*{proposition_1}{Proposition 1}
\newtheorem*{proposition_2}{Proposition 2}

\newtheorem{lemma}{Lemma}

\newenvironment{shrinkeq}[1]
{\bgroup
  \addtolength\abovedisplayshortskip{#1}
  \addtolength\abovedisplayskip{#1}
  \addtolength\belowdisplayshortskip{#1}
  \addtolength\belowdisplayskip{#1}}
{\egroup\ignorespacesafterend}

\title{Meta-Semi: A Meta-learning Approach for Semi-supervised Learning}

\author{%
  Yulin Wang, Jiayi Guo, Shiji Song, Gao Huang\thanks{Corresponding author.}\\
  Department of Automation, Tsinghua University, Beijing, China\\
  Beijing National Research Center for Information Science and Technology (BNRist)\\
  \texttt{wang-yl19@mails.tsinghua.edu.cn, guojy821@gmail.com} \\
  \texttt{\{shijis, gaohuang\}@tsinghua.edu.cn}
}

\begin{document}

\maketitle

\vskip -0.15in
\begin{abstract}
  Deep learning based semi-supervised learning (SSL) algorithms have led to promising results in recent years. However, they tend to introduce multiple tunable hyper-parameters, making them less practical in real SSL scenarios where the labeled data is scarce for extensive hyper-parameter search. In this paper, we propose a novel meta-learning based SSL algorithm (\textit{Meta-Semi}) that requires tuning only one additional hyper-parameter, compared with a standard supervised deep learning algorithm, to achieve competitive performance under various conditions of SSL. We start by defining a meta optimization problem that minimizes the loss on labeled data through dynamically reweighting the loss on unlabeled samples, which are associated with soft pseudo labels during training. As the meta problem is computationally intensive to solve directly, we propose an efficient algorithm to dynamically obtain the approximate solutions. We show theoretically that \textit{Meta-Semi} converges to the stationary point of the loss function on labeled data under mild conditions. Empirically, \textit{Meta-Semi} outperforms state-of-the-art SSL algorithms significantly on the challenging semi-supervised CIFAR-100 and STL-10 tasks, and achieves competitive performance on CIFAR-10 and SVHN.\footnote[1]{This work has been submitted to the IEEE for possible publication. Copyright may be transferred without notice, after which this version may no longer be accessible.}
\end{abstract}

\vspace{-2ex}
\section{Introduction}
\vspace{-2ex}
The recent success of deep learning in supervised tasks is fueled by abundant annotated training data \cite{krizhevsky2012imagenet, simonyan2014very, szegedy2015going, lecun2015deep, He_2016_CVPR, huang2019convolutional}. However, collecting precise labels in practice is usually very time-consuming and costly. In many real-world applications, only a small subset of all available training data are associated with labels \cite{oliver2018realistic, verma2019interpolation}.
Semi-supervised learning (SSL) is a learning paradigm that aims to improve the model performance by simultaneously leveraging labeled and unlabeled data \cite{zhu2003semi, SSL_book, turian2010word}.



In the context of deep learning, many successful SSL methods incorporate unlabeled data by performing unsupervised consistency regularization \cite{laine2016temporal, tarvainen2017mean, miyato2018virtual, verma2019interpolation, berthelot2019mixmatch}. In specific, they first add small perturbations to the unlabeled samples, and then enforce the consistency between the model predictions on the original data and the perturbed data. Though impressive performance has been achieved, the state-of-the-art consistency based algorithms tend to introduce multiple tunable hyper-parameters. The final performance of the algorithms is usually conditioned on setting proper values for these hyper-parameters. However, in real semi-supervised learning scenarios, hyper-parameter searching is usually unreliable as the annotated data are scarce, leading to high variance when cross-validation is adopted~\cite{oliver2018realistic}. This problem will become even more serious if the performance of the algorithm is sensitive to the hyper-parameter values. Furthermore, since the searching space grows exponentially with respect to the number of hyper-parameters \cite{bergstra2012random}, the computational cost may become unaffordable for modern deep learning algorithms.

Another challenge to develop practical and robust deep SSL algorithms is how to exploit the \emph{labeled data} more efficiently, as these data, although being scarce, have the precise and reliable annotations. Consistency based SSL algorithms \cite{laine2016temporal, tarvainen2017mean, miyato2018virtual, verma2019interpolation, berthelot2019mixmatch} usually model the labeled and unlabeled data in separate terms in the loss function, where the unlabeled data receives no supervision, at least explicitly, from the former, leading to an inefficient use of the labeled data.

In this paper, we propose a meta-learning based SSL algorithm, named \textit{Meta-Semi}, to efficiently exploit the labeled data, while requiring tuning only one additional hyper-parameter to achieve impressive performance under various conditions. The proposed algorithm is based on a simple intuition: \textit{if the network is trained with correctly ``pseudo-labeled'' unannotated samples, the final loss on labeled data should be minimized}. 
To be specific, we start by explicitly defining a meta reweighting objective: finding the optimal weights\footnote{Throughout the paper, the term ``weights'' always refer to the coefficients that we use to reweight each individual unlabeled sample, instead of referring to the parameters of neural networks.} for different pseudo-labeled samples to train a network, such that the final loss on labeled data is minimized. Note that the problem is computationally intensive to be directly solved via optimization algorithms. Therefore, we propose an approximated formulation, based on which a closed form solution can be obtained. We show theoretically that one meta gradient step is sufficient to obtain the approximate solutions at each training iteration. Finally, we propose a dynamical weighting algorithm to reweight pseudo-labeled samples with 0-1 weights. Theoretical analysis shows that our method converges to the stationary point of the supervised loss function.

Our algorithm is empirically validated on widely used image classification benchmarks (CIFAR-10, CIFAR-100, SVHN and STL-10) with modern deep networks (e.g., CNN-13 and WRN-28). 
\textit{Meta-Semi} outperforms state-of-the-art SSL algorithms, including ICT \cite{verma2019interpolation} and MixMatch \cite{berthelot2019mixmatch}, on the challenging CIFAR-100 and STL-10 SSL tasks significantly, while achieves slightly better performance than them on CIFAR-10. 
Besides, \textit{Meta-Semi} is complementary to consistency based methods, i.e., performing consistency regularization in our algorithm further improves the performance. Moreover, sensitivity test on the only tunable hyper-parameter of \emph{Meta-Semi} shows that the algorithm is quite robust to different hyper-parameter values.

\vspace{-2ex}
\section{Related Work}
\vspace{-2ex}
\textbf{Consistency based semi-supervised learning} has been extensively studied in the context of deep learning in recent years \cite{sajjadi2016regularization, laine2016temporal, tarvainen2017mean, miyato2018virtual, verma2019interpolation}. These methods leverage unlabeled data by adding an unsupervised regularization term to the standard supervised loss: $\mathcal{L}_{S} + w\mathcal{L}_{US}$, where $\mathcal{L}_{S}$ is the conventional loss on labeled data, $\mathcal{L}_{US}$ is the loss contributed by unlabeled data which is usually defined as a measure of discrepancy between the model predictions on the original unlabeled samples and their perturbed counterparts, and $w$ is a pre-defined coefficient. Existing approaches have proposed different ways to generate the perturbations for $\mathcal{L}_{US}$, including data augmentation \cite{bachman2014learning, laine2016temporal, sajjadi2016regularization}, adversarial noise \cite{miyato2018virtual}, Dropout \cite{park2018adversarial}, data interpolation \cite{verma2019interpolation}, etc. To enhance the model stability, an exponential moving average (EMA) on parameters or predictions is often adopted \cite{laine2016temporal, tarvainen2017mean}. The effectiveness of these approaches is conditioned on the proper setting of the coefficient $w$. As the recent methods \cite{berthelot2019mixmatch, berthelot2019remixmatch} usually integrate multiple regularization techniques, finding the proper hyperparameter setting becomes a challenging problem in practice, especially in the SSL scenarios where few samples are available for performing cross-validation.

\textbf{Other semi-supervised learning algorithms.}
Early work on SSL can be categorized into cluster assumption based methods \cite{joachims2003transductive,joachims1999transductive} and graph assumption based methods \cite{zhu2003semi, bengio200611}. For deep learning based SSL, \cite{kingma2014semi, odena2016semi} propose to train deep generators using both the labeled and unlabeled data to estimate the data distribution. Pseudo label based method \cite{lee2013pseudo} is also widely used in deep SSL. It progressively uses the highly confident model predictions to generate pseudo labels for unlabeled samples during training. Minimizing the entropy of the model prediction on unlabeled data is also proven effective for SSL \cite{grandvalet2005semi, miyato2018virtual}.

\textbf{Meta learning.}
Since \textit{Meta-Semi} follows a meta-learning paradigm, we briefly review the existing work on this topic. The idea of meta-learning is motivated by the goal of `learning to learn better' \cite{lake2017building, andrychowicz2016learning}. Meta-learning algorithms usually define a meta optimization problem to extract information from the learning process.
For example, using the loss on a small amount of trustable data as the meta-objective is widely adopted in few-shot learning \cite{ravi2016optimization, ren2018meta}. MAML \cite{finn2017model} proposes to minimize the meta loss directly via gradient descents.
To address the challenge that naively minimizing the meta objective requires performing multiple meta update steps iteratively for every `real' update step on model parameters, \cite{ren2018learning} propose an online approximation method to make the meta training process more tractable. The proposed algorithm is similar to that in \cite{ren2018learning}, but our contributions lie in several important aspects. First, we propose to exploit the labeled data more efficiently in SSL by leveraging the meta-reweighting method, which not only reduces the required number of tunable hyper-parameters, but also effectively improves the performance. As far as we know, this idea has not been explored in the literature. Second, we propose a novel dynamical re-weighting process that is tailored for SSL. This is non-trivial since directly applying the method in \cite{ren2018learning} to SSL leads to inferior results (see: Table 1). Third, we provide a theoretical convergence analysis in the context of SSL, which utilizes different proof techniques from \cite{ren2018learning}.

\vspace{-2ex}
\section{Method}
\vspace{-2ex}
In this section, we introduce the details of our \textit{Meta-Semi} algorithm.
Different from most existing methods that leverage unsupervised consistency regularization, we propose to solve the SSL problem in a meta-learning paradigm. 
As an overview, we first compute the cross-entropy loss of unlabeled samples using their corresponding pseudo labels. Then we reweight the loss on each unlabeled sample by solving a meta optimization problem that minimizes the supervised loss of labeled samples.
As directly solving the meta problem is computationally intractable, we propose an approximation method to dynamically obtain the 0-1 approximate solutions, which only requires one meta gradient descent step. In addition, theoretical guarantees are provided to show that our method converges to the stationary point of the supervised loss.





\vspace{-2ex}
\subsection{Meta Optimization Problem}
\vspace{-2ex}
\label{sec:Overview}
We start by presenting the weighted loss function of our method, and defining a meta optimization problem to determine the value of the weight for each unlabeled sample.


Suppose that the networks are trained with stochastic gradient descent (SGD). 
At each iteration, we sample a mini-batch of labeled samples $\mathcal{X}\!=\!\{({\bm{x}}_i, {\bm{y}}_i)\}$ together with a mini-batch of unlabeled samples $\mathcal{U}\!=\!\{({\bm{u}}_j, \hat{\bm{y}}_j)\}$, where $\bm{x}_i$ and $\bm{y}_i$ represent the $i^{\text{th}}$ labeled sample and its associated ground truth label, respectively, and $\bm{u}_j$ and $\hat{\bm{y}}_j$ represent the $j^{\text{th}}$ unlabeled sample and its pseudo label, respectively. 
Following earlier work \cite{verma2019interpolation, berthelot2019mixmatch}, we use the MixUp augmentation \cite{zhang2017mixup} to generate a mixed version of the inputs to improve the generalization performance, instead of directly using $\mathcal{X}$ and $\mathcal{U}$. The augmented mini-batch of training samples are denoted by $\tilde{\mathcal{X}}\!=\!\{(\tilde{\bm{x}}_i, \tilde{\bm{y}}_i)\}$ and $\tilde{\mathcal{U}}\!=\!\{(\tilde{\bm{u}}_j, \hat{\bm{y}}_j)\}$. 
We defer the details on generating pseudo labels and obtaining $\tilde{\mathcal{X}}$ and $\tilde{\mathcal{U}}$ to Section \ref{sec:Implementation Details}.

Consider training a deep network with parameters $\bm{\theta}$. 
We first feed an unlabeled sample $\tilde{\bm{u}}_j$ into the network, producing its prediction $p(\tilde{\bm{u}}_j|\bm{\theta})$. Then we calculate the cross-entropy loss $L(\hat{\bm{y}}_j, p(\tilde{\bm{u}}_j|\bm{\theta}))$ using the corresponding soft pseudo label $\hat{\bm{y}}_j$. The loss of this sample is further reweighed by $w_j^*\in[0,1]$ to construct the final loss function
\begin{shrinkeq}{-0.5ex}
    \begin{equation}
        \mathcal{L}_{meta} = \frac{1}{\sum_{j=1}^{|\tilde{\mathcal{U}}|}w_j^*} \sum\nolimits_{j=1}^{|\tilde{\mathcal{U}}|} w_j^* L(\hat{\bm{y}}_j, p(\tilde{\bm{u}}_j|\bm{\theta})).
    \end{equation}
\end{shrinkeq}
Without loss of generality, we assume $\mathcal{L}_{meta} = 0$ when $\sum_{j=1}^{|\tilde{\mathcal{U}}|}w_j^* = 0$.
The weight scalar $w_j^*$ is determined by minimizing the meta loss on the labeled data. To illustrate that, we first consider training the network with a similar weighted loss
\begin{shrinkeq}{-0.5ex}
    \begin{equation}
    \label{eq:meta_objective_1}
    \bm{\theta}^{*}(\bm{w}) = \mathop{\arg\min}_{\bm{\theta}} \sum\nolimits_{j=1}^{|\tilde{\mathcal{U}}|} w_j L(\hat{\bm{y}}_j, p(\tilde{\bm{u}}_j|\bm{\theta})),
\end{equation}
\end{shrinkeq}
where $\bm{\theta}^{*}(\bm{w})$ is the optimal solution that minimizes the weighted loss. Obviously, it is a function of the weight vector $\bm{w} = [w_1, w_2, \ldots]^T$. Then the weights $\bm{w}^*$ is solved by minimizing the loss on labeled data $\tilde{\mathcal{X}}$ with $\bm{\theta}^{*}(\bm{w})$, namely
\begin{shrinkeq}{-0.5ex}
    \begin{equation}
    \label{eq:meta_objective_2}
    \bm{w}^{*} = \mathop{\arg\min}_{w_j\in[0,1], j=1,\ldots, |\tilde{\mathcal{U}}|}\sum\nolimits_{i=1}^{|\tilde{\mathcal{X}}|}L(\tilde{\bm{y}}_i, p(\tilde{\bm{x}}_i|\bm{\theta}^{*}(\bm{w}))).
\end{equation}
\end{shrinkeq}
Intuitively, our aim is to find a subset of pseudo-labeled samples, which, if used for training, are the most beneficial in terms of the generalization performance. The labeled data are leveraged to determine if each pseudo-labeled sample should be used, instead of directly being used for train as most existing SSL algorithms do \cite{laine2016temporal, tarvainen2017mean, miyato2018virtual, verma2019interpolation, berthelot2019mixmatch}. We argue that this is a more effective approach to exploit the supervision information.

\vspace{-2ex}
\subsection{Approximating the Meta Solution}
\vspace{-2ex}
\label{sec:Approximating the Meta Objective}
To solve the meta optimization problem Eqs. (\ref{eq:meta_objective_1}) and (\ref{eq:meta_objective_2}) efficiently, we introduce a method to obtain an approximate solution. 


At $t^\text{th}$ step in the training process, consider estimating $\bm{\theta}^{*}(\bm{w})$ by performing $M$ times of gradient descents starting from current values of network parameters $\bm{\theta}^{t}$:
\begin{shrinkeq}{-1ex}
    \begin{equation}
    \label{eq:update_theta_1_M}
    \overline{\bm{\theta}}^t_M \approx\bm{\theta}^{*}(\bm{w}),\ \ \overline{\bm{\theta}}^t_0=\bm{\theta}^{t},
\end{equation}
\end{shrinkeq} 
\begin{shrinkeq}{-1.5ex}
    \begin{equation}
        \label{eq:update_theta}
    \overline{\bm{\theta}}^t_{m+1} = \overline{\bm{\theta}}^t_{m} - \alpha^t \!\! \left[ 
        \frac{\partial\sum_{j=1}^{|\tilde{\mathcal{U}}|} w_j L(\hat{\bm{y}}_j, p(\tilde{\bm{u}}_j|\overline{\bm{\theta}}^t_{m}))}{\partial \overline{\bm{\theta}}^t_{m}} 
    \right], m=0,1, \ldots ,M-1,
\end{equation}
\end{shrinkeq} 
where $\alpha^t\!$ is the learning rate. As SGD has proven to be effective for optimizing deep networks, $\overline{\bm{\theta}}^t_M$ is a reliable alternate of $\bm{\theta}^{*}(\bm{w})$ as long as $M$ is sufficiently large.
\vskip 0.03in

Given that $\bm{\theta}^{*}(\bm{w})$ can be estimated by $\overline{\bm{\theta}}^t_M$, a naive method of approximating $\bm{w}^{*}$ is to further estimate the gradient $\nabla_{\!\bm{w}} \!\sum_{i=1}^{|\tilde{\mathcal{X}}|}\!L(\tilde{\bm{y}}_i, p(\tilde{\bm{x}}_i|\bm{\theta}^{*}(\bm{w})))$ with $\overline{\bm{\theta}}^t_M$, and then repeatedly update $\bm{w}$ following similar gradient based optimization algorithms. However, it is computationally intensive to do that since updating $\bm{w}$ for $N$ times requires $MN$ steps of gradient descents on the network parameters. To get a efficient estimate of $\bm{w}^{*}$, we propose a dynamic approximation approach in the following. 
\vskip 0.07in


First, to reduce the iterations of updating $\bm{w}$, we exploit a first order Taylor approximation of Eq. (\ref{eq:meta_objective_2}) at $\bm{w}=0$:
\begin{shrinkeq}{-1ex}
    \begin{equation}
    \label{eq:meta_objective_2_est}
    \bm{w}^{*}\!\!\approx\!\!\!\mathop{\arg\min}_{w_j\in[0,1], j=1,\ldots, |\tilde{\mathcal{U}}|}\!\!\!\bm{w}^T \!\!
    \left[
        \left.
        \!\frac{\partial\sum_{i=1}^{|\tilde{\mathcal{X}}|}L(\tilde{\bm{y}}_i, p(\tilde{\bm{x}}_i|\overline{\bm{\theta}}^t_M))}{\partial {\bm{w}}}
        \right|_{\bm{w}=0}
    \right]
        \!.
\end{equation}
\end{shrinkeq}
Notably, $\overline{\bm{\theta}}^t_M$ is obtained using the gradients of the weighted loss according to Eq. (\ref{eq:update_theta}), and thus it is differentiable with respect to $w_j$. 
As the optimization objective in Eq. (\ref{eq:meta_objective_2_est}) is linear, it is straightforward to derive the solution:
\begin{shrinkeq}{-0.6ex}
    \begin{equation}
    \label{eq:taylor_est}
    w_j^* \approx w_j^t =
    \begin{cases}
        1
        & 
        \left. \frac{\partial\sum_{i=1}^{|\tilde{\mathcal{X}}|}L(\tilde{\bm{y}}_i, p(\tilde{\bm{x}}_i|\overline{\bm{\theta}}^t_M))}{\partial w_j}\right|_{\bm{w}=0} \leq 0 \vspace{1ex} \\
        0
        & 
        \left.\frac{\partial\sum_{i=1}^{|\tilde{\mathcal{X}}|}L(\tilde{\bm{y}}_i, p(\tilde{\bm{x}}_i|\overline{\bm{\theta}}^t_M))}{\partial w_j}\right|_{\bm{w}=0} > 0
    \end{cases},
\end{equation}
\end{shrinkeq}
where $w_j^t$ denotes the approximate solution of $w_j^*$. The required steps of gradient descents are reduced to $M$ from $MN$ by leveraging Eq. (\ref{eq:taylor_est}). However, the algorithm is still inefficient since a large $M$ is necessary to get a sufficiently accurate $\overline{\bm{\theta}}^t_M$. To further reduce the computational cost, an intriguing property can be leveraged. In the following proposition, we show that the results of Eq. (\ref{eq:taylor_est}) will remain the same if $\overline{\bm{\theta}}^t_M$ in the equation is replaced by $\overline{\bm{\theta}}^t_1$. In other words, Eq. (\ref{eq:taylor_est}) can be precisely solved using $\overline{\bm{\theta}}^t_1$ instead of $\overline{\bm{\theta}}^t_M$, and the former only needs one gradient descent step to obtain.
\vskip 0.12in

\begin{proposition}
    \label{prop:T_1}
    Suppose that $\overline{\bm{\theta}}^t_M$ is given by $M$ steps of gradient descents starting from $\overline{\bm{\theta}}^t_0=\bm{\theta}^{t}$. Then we have
    \begin{shrinkeq}{-1ex}
    \begin{equation}
            \left. \frac{\partial\sum_{i=1}^{|\tilde{\mathcal{X}}|}L(\tilde{\bm{y}}_i, p(\tilde{\bm{x}}_i|\overline{\bm{\theta}}^t_M))}{\partial w_j}\right|_{\bm{w}=0} 
        =  \left. M \left[  \frac{\partial\sum_{i=1}^{|\tilde{\mathcal{X}}|}L(\tilde{\bm{y}}_i, p(\tilde{\bm{x}}_i|\overline{\bm{\theta}}^t_1))}{\partial w_j}\right|_{\bm{w}=0}  \right]\!\!,\ \forall\  1\!\leq\!j\!\leq |\tilde{\mathcal{U}}|.
    \end{equation}
\end{shrinkeq}
\end{proposition}
\begin{proof}
    \vskip -0.1in
    See Appendix A.
\end{proof}
\vskip -0.15in
With Proposition \ref{prop:T_1}, we are ready to present the final form of our dynamically reweighting formula:
\begin{shrinkeq}{-0.8ex}
    \begin{equation}
    \label{eq:approximation_formula}
    w_j^t =
    \begin{cases}
        1
        & 
        \left. \frac{\partial\sum_{i=1}^{|\tilde{\mathcal{X}}|}L(\tilde{\bm{y}}_i, p(\tilde{\bm{x}}_i|\overline{\bm{\theta}}^t_1))}{\partial w_j}\right|_{\bm{w}=0} \leq 0 \vspace{1ex} \\
        0
        & 
        \left.\frac{\partial\sum_{i=1}^{|\tilde{\mathcal{X}}|}L(\tilde{\bm{y}}_i, p(\tilde{\bm{x}}_i|\overline{\bm{\theta}}^t_1))}{\partial w_j}\right|_{\bm{w}=0} > 0
    \end{cases}.
\end{equation}
\end{shrinkeq}
As we leverage a meta learning approach to reweight different pseudo-labeled samples, we call our method \textit{Meta-Semi}. The pseudo code of \textit{Meta-Semi} is presented in Algorithm \ref{alg}. In summary, after each standard forward step of the pseudo-labeled samples, we first update the parameters with the loss of all samples weighted by zero. Such a meta updating step does not change the values of parameters, but construct a differentiable computational graph. Then we calculate the supervised loss on labeled data, and exploit the computational graph to take the derivative of the supervised loss with respect to the zero weight, which is called ``meta gradient''. Finally, we only use the pseudo-labeled samples with negative meta gradients to train the network. 


\begin{figure}
    \begin{minipage}{0.48\columnwidth}
        \footnotesize
        \vspace{-2.5ex}
        \begin{center}
    \begin{algorithm}[H]
        \algsetup{linenosize=\tiny} \footnotesize
        \caption{\footnotesize The Meta-Semi Algorithm.}
        \label{alg}
    \begin{algorithmic}[1]
        \STATE {\bfseries Initialize:} $\bm{\theta}^{0}$
        \FOR{$t=1$ {\bfseries to} $T$}
        \STATE Randomly sample ${\mathcal{X}}$, ${\mathcal{U}}$
        \STATE Generate $\tilde{\mathcal{X}}$, $\tilde{\mathcal{U}}$
        \STATE Compute $p(\tilde{\bm{u}}_j|\bm{\theta}^t)$,\ $\tilde{\bm{u}}_j\!\in\tilde{\mathcal{U}}$
        \STATE ${\bm{w}} \leftarrow 0$,  $ \overline{\bm{\theta}}^t_0 \leftarrow \bm{\theta}^{t}$
        \STATE $\nabla_{\overline{\bm{\theta}}^t_0} \leftarrow 
        \frac{\partial\sum_{j=1}^{|\tilde{\mathcal{U}}|} {w}_j L(\hat{\bm{y}}_j, p(\tilde{\bm{u}}_j|\bm{\theta}^t))}{\partial \bm{\theta}^t}$
        \STATE $\overline{\bm{\theta}}^t_{1} \leftarrow \overline{\bm{\theta}}^t_{0} - \alpha^t \nabla_{\overline{\bm{\theta}}^t_0}$
        \STATE Compute $p(\tilde{\bm{x}}_i|\overline{\bm{\theta}}^t_{1})$,\ $\tilde{\bm{x}}_i\!\in\tilde{\mathcal{X}}$
        \STATE {\bfseries Meta Gradient:} $\nabla_{{\bm{w}}}^t\!\!\leftarrow\!\!
        \left. \frac{\partial\sum_{i=1}^{|\tilde{\mathcal{X}}|}L(\tilde{\bm{y}}_i, p(\tilde{\bm{x}}_i|\overline{\bm{\theta}}^t_1))}{\partial {\bm{w}}}\right.$
        \STATE $\bm{w}^t \leftarrow sign(max(-\nabla_{{\bm{w}}}^t, 0))$\ \ \ (Eq. (\ref{eq:approximation_formula}))
        \STATE $\mathcal{L}_{meta}\!\!\leftarrow\!\!\frac{1}{\sum_{j=1}^{|\tilde{\mathcal{U}}|}w_j^t} \sum_{j=1}^{|\tilde{\mathcal{U}}|} w_j^t L(\hat{\bm{y}}_j, p(\tilde{\bm{u}}_j|\bm{\theta^{t}}))$ \vspace{0.5ex}
        \STATE $\bm{\theta}^{(t+1)} \leftarrow \bm{\theta}^{t} - \alpha^t \frac{\partial \mathcal{L}_{meta}}{\partial \bm{\theta}^t}$
        \ENDFOR
    \end{algorithmic}
    \end{algorithm}
    \end{center}
    \end{minipage}
    \hspace{0.05in}
    \begin{minipage}{0.5\columnwidth}
        \centering
        \includegraphics[width=\columnwidth]{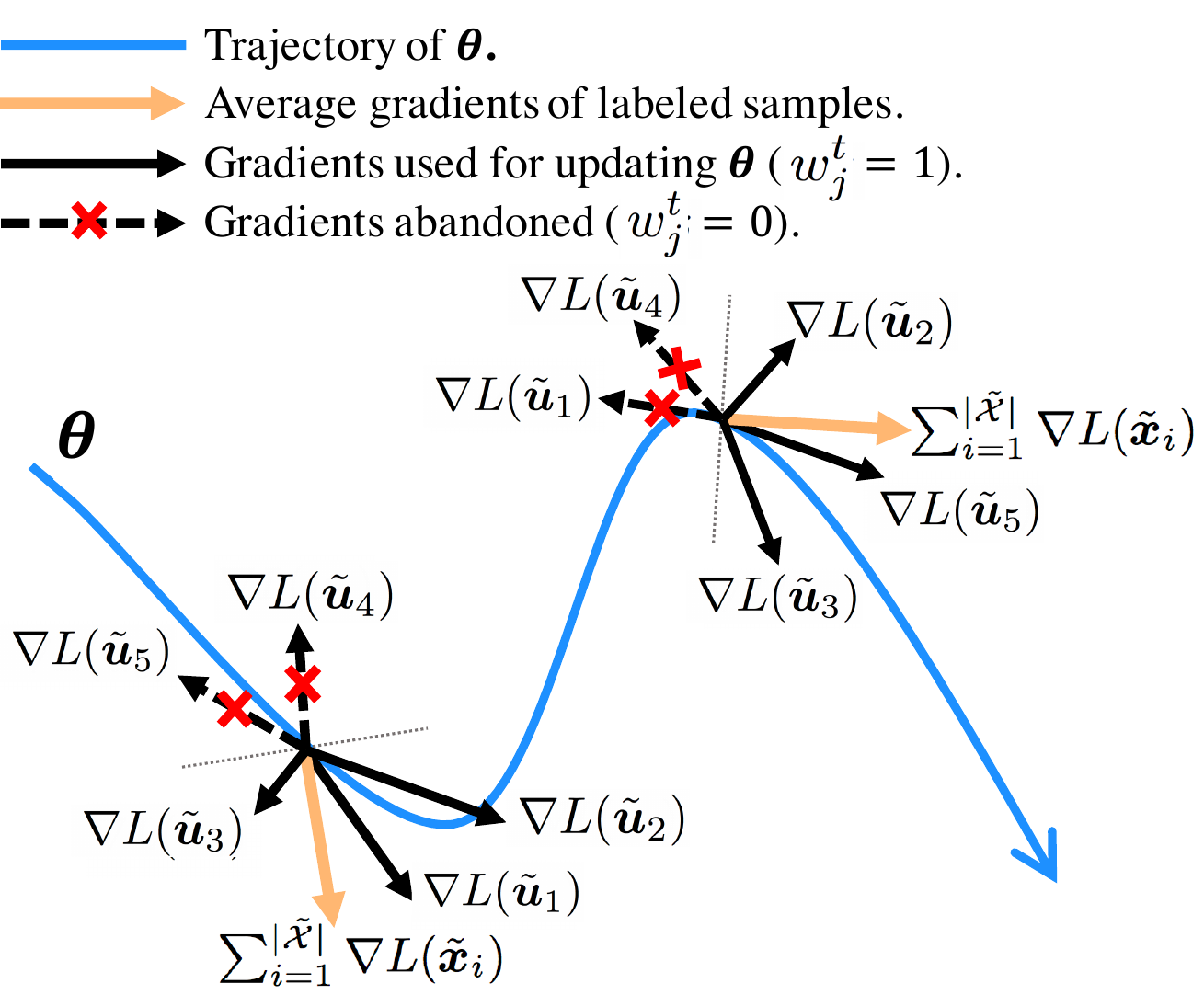}	
        \vskip -0.1in
        \caption{Illustration of \textit{Meta-Semi}. Herein, $\nabla L(\tilde{\bm{u}}_j)$ and $\nabla L(\tilde{\bm{x}}_i)$ denote $\nabla_{\!\bm{\theta}^t}{L(\hat{\bm{y}}_j, p(\tilde{\bm{u}}_j|\bm{\theta}^t))}$ and $\nabla_{\!\bm{\theta}^t}{L(\tilde{\bm{y}}_i, p(\tilde{\bm{x}}_i|\bm{\theta}^t))}$, respectively. Our method trains the networks with pseudo-labeled samples whose gradient directions are similar to the average gradient of labeled samples. }
        \label{fig:illustration}
        \vskip -0.1in  
    \end{minipage}
    \vskip -0.2in
\end{figure}

\textbf{Interpretation of meta gradients. }
A straightforward way to interpret the meta gradients is that it can be viewed as the influence on the supervised loss when the weight of certain pseudo-labeled sample changes slightly around zero during training. In fact, there exists a more intriguing and interesting interpretation. The meta gradients given in Eq. (\ref{eq:approximation_formula}) can be expressed as
\begin{shrinkeq}{-0.2ex}
    \begin{equation}
    \label{eq:interpretion}
    \begin{split}
         \left. \frac{\partial\sum_{i=1}^{|\tilde{\mathcal{X}}|}L(\tilde{\bm{y}}_i, p(\tilde{\bm{x}}_i|\overline{\bm{\theta}}^t_1))}{\partial w_j}\right|_{\bm{w}=0}
        \!= &\! \left[\!\left.\frac{\partial \!\sum_{i=1}^{|\tilde{\mathcal{X}}|} \!L(\tilde{\bm{y}}_i, p(\tilde{\bm{x}}_i|\overline{\bm{\theta}}^t_1))}{\partial \overline{\bm{\theta}}^t_1} \! \right]^{\!T} \! \!
        \left[ \!\frac{\partial (\overline{\bm{\theta}}^t_0\!-\!\alpha^t \nabla_{\overline{\bm{\theta}}^t_0} ) }{ \partial w_j} \! \right] \right|_{\bm{w}=0} \\
        \!= &\!  -\! \alpha^t \!\! \left[ \!\frac{\partial \!\sum_{i=1}^{|\tilde{\mathcal{X}}|}\!L(\tilde{\bm{y}}_i, p(\tilde{\bm{x}}_i|\bm{\theta}^t))}{\partial \bm{\theta}^t} \! \right]^{\!T} \! \!
        \left[ \!\frac{\partial \!L(\hat{\bm{y}}_j, p(\tilde{\bm{u}}_j|\bm{\theta}^t))}{\partial \bm{\theta}^t} \!
            \right]\!,
    \end{split}
\end{equation}
\end{shrinkeq} 
which follows from $\nabla_{\overline{\bm{\theta}}^t_0}\!\!=\!\!\sum_{k=1}^{|\tilde{\mathcal{U}}|}\! \!w_k\!\frac{\partial \!L(\bm{\tilde{y}}_k, p(\tilde{\bm{u}}_k|\overline{\bm{\theta}}^t_0))}{\partial \overline{\bm{\theta}}^t_0} \!$  and $\overline{\bm{\theta}}^t_1 = \overline{\bm{\theta}}^t_0 = \bm{\theta}^t$. For the pseudo-unlabeled sample $(\tilde{\bm{u}}_j, \hat{\bm{y}}_j)$, its meta gradient is negatively proportional to the inner product of the average gradient of labeled samples and the gradient produced by itself. In other words, the sign of the meta gradient indicates whether the angle between the former and the later is larger than 90 degrees. Intuitively, if the pseudo label is correct, the corresponding gradient should guide the model towards a similar direction to the labeled samples, or at least should not be largely different from the supervised gradient  in direction. In essence, \textit{Meta-Semi} trains deep networks using pseudo-labeled samples whose gradient directions are similar to labeled samples. An illustration is shown in Figure \ref{fig:illustration}.
 
\vspace{-2ex}
\subsection{Implementation Details}
\vspace{-2ex}
\label{sec:Implementation Details}
\textbf{Pseudo labels.} 
To obtain high quality pseudo labels for the original unlabeled mini-batch $\mathcal{U}$, we first apply an exponential moving average (EMA) on model parameters, which has proven to be effective in providing supervision on unlabeled data \cite{tarvainen2017mean, verma2019interpolation}. Then we feed every unlabeled sample $\bm{u}_j$ in $\mathcal{U}$ into the EMA model, and take the corresponding softmax prediction as the soft pseudo label $\hat{\bm{y}}_j$.

\textbf{MixUp augmentation }
is an important regularization technique used by state-of-the-art deep SSL algorithms \cite{verma2019interpolation, berthelot2019mixmatch}. It improves the generalization performance of models by encouraging the `convex' behavior between different samples. Given a pair of samples with corresponding annotations ($\bm{x}_{1}$, $\bm{y}_{1}$) and ($\bm{x}_{2}$, $\bm{y}_{2}$), MixUp is performed to generate an augmented sample via linear interpolation:
\begin{shrinkeq}{-0ex}
\begin{align}
    \tilde{\bm{x}}=\lambda \bm{x}_{1}+(1-\lambda) \bm{x}_{2}, \ \ \ \ 
    \tilde{\bm{y}}=\lambda \bm{y}_{1}+(1-\lambda) \bm{y}_{2},
\end{align}
\end{shrinkeq}
where $\lambda$ is sampled from a pre-defined Beta distribution.
In \textit{Meta-Semi}, we leverage MixUp to generate the mixed training data $\tilde{\mathcal{X}}$ and $\tilde{\mathcal{U}}$. Formally, $\tilde{\mathcal{X}}$ is obtained from only the labeled set $\mathcal{X}$:
\begin{shrinkeq}{-0ex}
    \begin{equation}
    \label{eq:mixup_1}
    \tilde{\mathcal{X}} = \text{MixUp}(\mathcal{X}, \text{Shuffle}(\mathcal{X}), \lambda_1), \ \ \ \lambda_1\!\sim\!\text{Beta}(\beta, \beta),
\end{equation}
\end{shrinkeq}
where $\beta$ is the parameter of the Beta distribution, and it is the only tunable hyper-parameter (excluding the hyper-parameters of a supervised learning algorithm) in our algorithm. With regards to $\tilde{\mathcal{U}}$, we ideally want the unlabeled data to extract more information from the labeled samples. Therefore, we first concatenate $\mathcal{X}$ and $\mathcal{U}$ together, and then apply the MixUp procedure:
\begin{shrinkeq}{-0ex}
    \begin{equation}
    \label{eq:mixup_2}
    \tilde{\mathcal{U}} = \text{MixUp}({\mathcal{W}}, \text{Shuffle}({\mathcal{W}}), \lambda_2), \ \ \ 
    {\mathcal{W}} = \text{Concat}(\mathcal{X}, \mathcal{U}), \ \ \ \lambda_2\!\sim\!\text{Beta}(\beta, \beta),
\end{equation}
\end{shrinkeq}
where the one-hot ground truth labels are used for $\mathcal{X}$ and the soft pseudo labels are used for $\mathcal{U}$.

\textbf{Compatibility with consistency based methods.}
As a matter of fact, \textit{Meta-Semi} is compatible with existing consistency based algorithms, and they can be integrated when necessary. To see this, the regularization term can be simply appended to the loss function with an addition coefficient $w$:
\begin{shrinkeq}{-0.2ex}
    \begin{equation}
    \mathcal{L} = \mathcal{L}_{meta} + w \mathcal{L}_{consistency}.
\end{equation}
\end{shrinkeq}
In experiments, we show that although \textit{Meta-Semi} has already achieved state-of-the-art performance, its performance is still able to be significantly improved by integrating consistency regularization.

\begin{figure*}[t]
    \centering
    \includegraphics[width=\columnwidth]{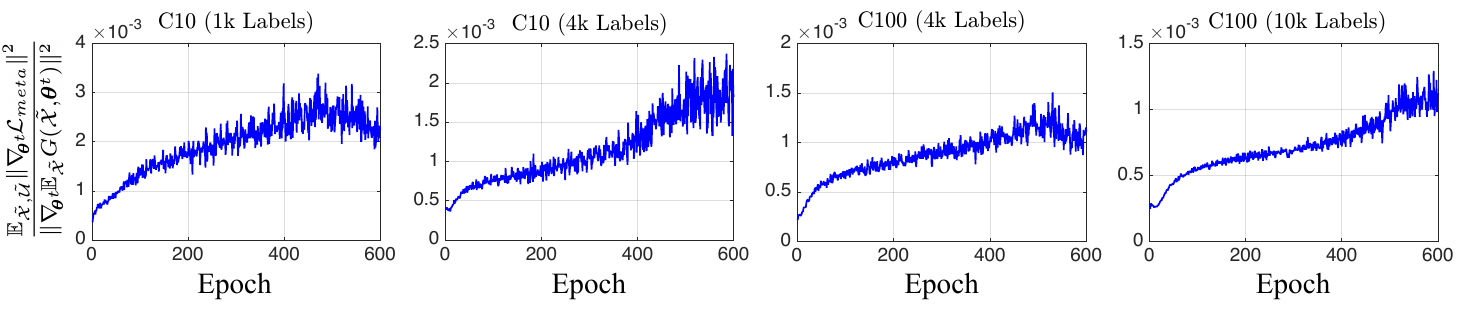}	
    \vskip -0.22in
    \caption{
        The empirical validation of Assumption \ref{assum:bound}. 
        The value of $\frac{\mathbb{E}_{\tilde{\mathcal{X}}, \tilde{\mathcal{U}}} \lVert  \nabla_{\!\!\bm{\theta}^t}\!\mathcal{L}_{meta} \lVert^2}{\lVert \nabla_{\!\!\bm{\theta}^t}\!\mathbb{E}_{\tilde{\mathcal{X}}} G(\tilde{\mathcal{X}}, \bm{\theta}^t) \lVert^2}$ is estimated at each training epoch using Monte-Carlo sampling with a sample size $500$. Results on CIFAR-10 (C10) and CIFAR-100 (C100) with varying numbers of labeled samples are presented. It can be observed that the ratio generally increases before the $500^{\text{th}}$ epoch, but gradually becomes stable or even decreases in the last part of the training process when the learning rate approaches 0. Therefore, it is empirically reasonable to assume that Assumption \ref{assum:bound} holds.
    }\label{fig:validation}
    \vskip -0.12in  
\end{figure*}

\vspace{-2ex}
\subsection{Convergence Analysis}
\vspace{-2ex}
\label{sec:Analysis of Convergence}

In this section, we show theoretically that under some mild conditions, our method converges to the stationary point of the loss on labeled data. The convergence results of SGD based optimization methods with a fixed loss function has been well-known \cite{reddi2016stochastic}. However, it is still necessary to provide the convergence analysis of \textit{Meta-Semi} since the optimization objective of our method is dynamically changed. To make it clear, we first define the supervised loss on the labeled mini-batch $\tilde{\mathcal{X}}$ by 
\begin{shrinkeq}{-0.7ex}
    \begin{equation}
    G(\tilde{\mathcal{X}}, \bm{\theta}^t) = \sum\nolimits_{i=1}^{|\tilde{\mathcal{X}}|} L(\tilde{\bm{y}}_i, p(\tilde{\bm{x}}_i|\bm{\theta}^t)).
\end{equation}
\end{shrinkeq}
Thus, the expected loss on all the labeled data is $\mathbb{E}_{\tilde{\mathcal{X}}} G(\tilde{\mathcal{X}}, \bm{\theta}^t).$
Then we introduce the definition of Lipschitz-smooth and a mild assumption stating that the expected norm of gradients used for updating model parameters will not get too large compared with the gradient of the overall supervised loss.

\begin{definition}
    A function $f\!:\! \mathbb{R}^n\!\!\!\to\!\!\mathbb{R}$ is said to be Lipschitz-smooth with constant $L$ if
\begin{align*}
\lVert \nabla f(x) - \nabla f(y) \rVert \leq L \lVert x - y \rVert,\ \  \forall x, y \in \mathbb{R}^n.
\end{align*}
\end{definition}
\begin{assumption}
    \label{assum:bound}
    For all $t \geq 0$, there exists a positive scalar $\sigma$, such that
    \begin{align*}
        \mathbb{E}_{\tilde{\mathcal{X}}, \tilde{\mathcal{U}}} \lVert  \nabla_{\!\!\bm{\theta}^t} \mathcal{L}_{meta} \lVert^2 \leq \sigma\lVert \nabla_{\!\!\bm{\theta}^t} \mathbb{E}_{\tilde{\mathcal{X}}} G(\tilde{\mathcal{X}}, \bm{\theta}^t) \lVert^2.
    \end{align*}
    \vskip -0.3in
\end{assumption}
\vskip -0.15in
In fact, the assumption is not very strong. Roughly, since $\mathcal{L}_{meta}$ is computed using the ground truth labels and the pseudo labels based on the prediction of the EMA model, it is usually very close to the minima of the loss function, especially when the networks tend to be stable with sufficiently large $t$. Empirically, we show that Assumption \ref{assum:bound} holds in many cases of SSL, which is shown in Figure \ref{fig:validation}. Under this condition, the following proposition shows that our method converges to the stationary point of the loss on labeled data with proper learning rate schedules.

\begin{proposition}
    \label{prop:convergence}
    Assume that the loss function on labeled data $G(\tilde{\mathcal{X}}, \bm{\theta}^t)$ is Lipschitz-smooth with regards to $\bm{\theta}^t$ for all $\tilde{\mathcal{X}}$, and that Assumption \ref{assum:bound} holds. Suppose also that the learning rate $\alpha^t > 0$ satisfies:
    \vspace{-1.2ex}
    \begin{shrinkeq}{0ex}
    \begin{equation}
        \lim_{t \to \infty} \alpha^t = 0, \ \ \ \ \sum_{t=0}^{\infty} \alpha^t = \infty.
    \end{equation}
\end{shrinkeq}
    \vskip -0.15in
    Then every limit point of the sequence $\{ \bm{\theta}^t \}$ generated by Meta-Semi is a stationary point of $\mathbb{E}_{\tilde{\mathcal{X}}} G(\tilde{\mathcal{X}}, \bm{\theta}^t)$, namely, 
    \vspace{-1ex}
    \begin{align*}
        \lim_{t \to \infty} \lVert \nabla_{\!\!\bm{\theta}^t} \mathbb{E}_{\tilde{\mathcal{X}}} G(\tilde{\mathcal{X}}, \bm{\theta}^t) \lVert = 0.
        \vspace{-2.5ex}
    \end{align*}
\end{proposition}
\begin{proof}
    \vspace{-2ex}
    See Appendix B.
\end{proof}
\vskip -0.2in


\begin{table*}[t]
    \centering
    \begin{scriptsize}
    \caption{Performance of \textit{Meta-Semi} and state-of-the-art SSL algorithms on CIFAR with varying amount of labeled data. We report the average test errors and the standard deviations of 5 trials. $\bm{\dagger}$ refers to the experiments using the WRN-28 network, while all others use the CNN-13 network. In each setting, the best two results with CNN-13 and the best result with WRN-28 are \textbf{bold-faced}.}
    \vskip -0.07in
    \label{tab:cnn13}
    \setlength{\tabcolsep}{0.75mm}{
    \renewcommand\arraystretch{1.2}
    \resizebox{\columnwidth}{!}{
    \begin{tabular}{l|ccc|cc}
    \hline
    Dataset & \multicolumn{3}{c|}{CIFAR-10}  &  \multicolumn{2}{|c}{CIFAR-100}     \\
    \hline
    Number of Labeled Samples & 1000&2000&4000  &  4000&10000    \\
    \hline

    Supervised & 39.95 $\pm$ 0.75\% & 27.67  $\pm$ 0.12\% & 20.42  $\pm$ 0.21\%  &  58.31  $\pm$ 0.89\%   &   44.56  $\pm$ 0.30\%     \\
    Supervised + MixUp \cite{zhang2017mixup} & 31.83  $\pm$ 0.65\% & 24.22  $\pm$ 0.15\% & 17.37  $\pm$ 0.35\%  &  54.87  $\pm$ 0.07\%   &   40.97  $\pm$ 0.47\%     \\

    \hline
    $\Pi$-model \cite{laine2016temporal} & 28.74  $\pm$ 0.48\% & 17.57  $\pm$ 0.44\%& 12.36  $\pm$ 0.17\% & 55.39  $\pm$ 0.55\%  & 38.06  $\pm$ 0.37\%    \\
    Temp-ensemble \cite{laine2016temporal}  & 25.15  $\pm$ 1.46\% &15.78  $\pm$ 0.44\% & 11.90  $\pm$ 0.25\%  & - &   38.65  $\pm$ 0.51\%    \\
    Mean Teacher \cite{tarvainen2017mean} & 18.27  $\pm$ 0.53\%&13.45  $\pm$ 0.30\%&10.73  $\pm$ 0.14\%  & 45.36  $\pm$ 0.49\%  & 35.96  $\pm$ 0.77\%   \\
    VAT \cite{miyato2018virtual} & 18.12   $\pm$ 0.82\%&13.93   $\pm$ 0.33\%&11.10   $\pm$ 0.24\%  & - & -     \\

    SNTG \cite{luo2018smooth}& 18.41   $\pm$ 0.52\%&13.64   $\pm$ 0.32\%&10.93   $\pm$ 0.14\%  & -  & 37.97  $\pm$ 0.29\%       \\

    Learning to Reweight \cite{ren2018learning} & 11.74 $\pm$ 0.12\% & - & 9.44 $\pm$ 0.17\% & 46.62 $\pm$ 0.29\% &  37.31 $\pm$ 0.47\%     \\

    MT + Fast SWA \cite{athiwaratkun2018there}& 15.58\%&11.02\%&9.05\%  & -  & 33.62 $\pm$ 0.54\%       \\

    ICT \cite{verma2019interpolation}&  12.44 $\pm$ 0.57\% &  8.69 $\pm$ 0.15\%  &  7.18 $\pm$ 0.24\%  & 40.07 $\pm$ 0.38\% &  32.24 $\pm$ 0.16\%  \\


    \textit{Meta-Semi} & \textbf{10.27 $\pm$ 0.66\%} &  \textbf{8.42 $\pm$ 0.30\%} &  \textbf{7.05 $\pm$ 0.27\%}  &  \textbf{37.61 $\pm$ 0.56\%}  &  \textbf{30.51 $\pm$ 0.32\%}      \\

    \textit{Meta-Semi} + ICT & \textbf{9.29 $\pm$ 0.62\%}  & \textbf{7.05 $\pm$ 0.12\%} & \textbf{6.42 $\pm$ 0.18\%}  & \textbf{37.12 $\pm$ 0.59\%}  & \textbf{29.68 $\pm$ 0.05\%}        \\

    \hline
    Mean Teacher $\bm{\dagger}$ \cite{tarvainen2017mean}  &  17.32 $\pm$ 4.00\% &  12.17 $\pm$ 0.22\%  &  10.36 $\pm$ 0.25\%  &  -  & - \\

    MixMatch $\bm{\dagger}$ \cite{berthelot2019mixmatch}  &  7.75 $\pm$ 0.32\% &  7.03 $\pm$ 0.15\%  & 6.24 $\pm$ 0.06\%  &  -  & 30.84 $\pm$ 0.29\% \\

    \textit{Meta-Semi} $\bm{\dagger}$ & \textbf{7.34 $\pm$ 0.22\%} & \textbf{6.58 $\pm$ 0.07\%} &  \textbf{6.10 $\pm$ 0.10\%}  & - & \textbf{29.69 $\pm$ 0.18\%}   \\


    \hline

    \end{tabular}}}
    \end{scriptsize}
    \vskip -0.22in
\end{table*}

\begin{wraptable}{r}{5.5cm}
	\small
    \centering
    \vskip -0.4in
    \caption{Test errors on STL-10. We adopt the same experimental setups as  \cite{berthelot2019mixmatch}. The best result is \textbf{bold-faced}.}
    \vskip -0.13in
    \label{tab:stl10}
    \setlength{\tabcolsep}{2.5mm}{
    \vspace{5pt}
    \renewcommand\arraystretch{1.1} 
    \begin{tabular}{l|c}
        \hline
            Method & STL-10, 1000 labels \\
            \hline
            SWWAE \cite{zhao2015stacked} & 25.70\% \\
            CC-GAN \cite{denton2016semi} & 22.20\% \\
            MixMatch \cite{berthelot2019mixmatch} & 10.18 $\pm$ 1.46\%  \\
            \textit{Meta-Semi} & \textbf{8.03 $\pm$ 0.24\%}\\
            \hline
    \end{tabular}}
    \vspace{-3ex}
\end{wraptable}
\section{Experiments}
\vspace{-2ex}
\label{sec:experiments}
In this section, we empirically evaluate the effectiveness of the proposed \textit{Meta-Semi} method, analyze its time complexity experimentally, and give sensitivity tests as well as ablation studies. All experiments are conducted using a single Nvidia Titan Xp GPU.  
\vspace{-2ex}
\subsection{Experimental Setup}
\vspace{-2ex}
Our experiments are based on four widely used image classification benchmarks, i.e., CIFAR-10/100 \cite{krizhevsky2009learning}, SVHN \cite{netzer2011reading} and STL-10 \cite{coates2011analysis}, and two modern deep networks, i.e., a 13-layer CNN (CNN-13) and the Wide-ResNet-28 (WRN-28). On CIFAR and SVHN, we randomly preserve the labels of certain numbers of samples (identical for each class), and remain all other samples unlabeled. On STL-10, we use pre-defined folds. Due to spatial limitation, details on data pre-processing, training/validation splitting, training configurations and baselines are deferred to Appendix C. These settings follow the common practice of SSL \cite{oliver2018realistic, berthelot2019mixmatch, verma2019interpolation, tarvainen2017mean, athiwaratkun2018there}. The hyper-parameter $\beta$ of \textit{Meta-Semi} is selected among $[0.2, 1]$ on the validation set.

\vspace{-2ex}
\subsection{Main Results}
\vspace{-2ex}
\textbf{Results on CIFAR} with various numbers of labeled samples are presented in Table \ref{tab:cnn13}. It can be observed that \textit{Meta-Semi} consistently outperforms state-of-the-art SSL algorithms in terms of generalization performance, especially with relatively less labeled data and larger numbers of classes. For example, when using CNN-13, on CIFAR-10 with 4000 labels, \textit{Meta-Semi} outperforms the competitive baseline, ICT, by $0.13\%$ in absolute error, while with 1,000 labels on CIFAR-10 and with 4,000 labels on CIFAR-100, \textit{Meta-Semi} yields more significant improvements of $2.17\%$ and $2.46\%$, respectively. MixMatch shows robust performance with small labeled sets as well, but \textit{Meta-Semi} outperforms it in terms of test accuracy. Moreover, it is shown that the performance of \textit{Meta-Semi} can be significantly improved by combining it with consistency based methods. On CIFAR-10 with 2000 labels, \textit{Meta-Semi} + ICT outperforms \textit{Meta-Semi} by $1.37\%$. 


\textbf{Results on STL-10 and SVHN} are presented in Table \ref{tab:stl10} and Table \ref{tab:svhn}, respectively. The results indicate that the test accuracy of \textit{Meta-Semi} outperforms MixMatch by more than $2\%$ on STL-10, and is comparable with state-of-the-art SSL algorithms on SVHN.

\begin{figure}[t]
    \begin{minipage}{0.49\columnwidth}
        \centering
    \begin{footnotesize}
    \captionof{table}{Test errors on SVHN with varying amount of labeled data. We report the average results and the standard deviations of 5 independent experiments. All results are based on CNN-13. The best results are \textbf{bold-faced}.}
    \vskip -0.05in 
    \label{tab:svhn}
    \setlength{\tabcolsep}{1.75mm}{
    \renewcommand\arraystretch{1.2}
    \resizebox{\columnwidth}{!}{
    \begin{tabular}{l|cc}
    \hline
    \multirow{2}{*}{Methods} & SVHN & SVHN     \\
     & 500 labels & 1000 labels \\
    \hline
    VAT \cite{miyato2018virtual} & - & 5.42\% \\
    $\Pi$-model \cite{laine2016temporal} & 6.65 $\pm$ 0.53\%& 4.82 $\pm$ 0.17\% \\
    Temp-ensemble \cite{laine2016temporal}  & 5.12 $\pm$ 0.13\% & 4.42 $\pm$ 0.16\% \\
    Mean Teacher \cite{tarvainen2017mean} & 4.18 $\pm$ 0.27\% &  3.95 $\pm$ 0.19\% \\
    ICT \cite{verma2019interpolation} & 4.23 $\pm$ 0.15\% & 3.89 $\pm$ 0.04\%\\
    SNTG \cite{luo2018smooth} & 3.99 $\pm$ 0.24\% & 3.86 $\pm$ 0.27\%\\
    \textit{Meta-Semi}  & 4.12 $\pm$ 0.21\% & 3.92 $\pm$ 0.11\% \\
    \textit{Meta-Semi} + ICT & \textbf{3.98 $\pm$ 0.09\%} & \textbf{3.77 $\pm$ 0.05\%}\\
    \hline
    \end{tabular}}}
    \end{footnotesize}
    \end{minipage}
    \hspace{0.02in}
    \begin{minipage}{0.49\columnwidth}
        \centering
        \includegraphics[width=1\columnwidth]{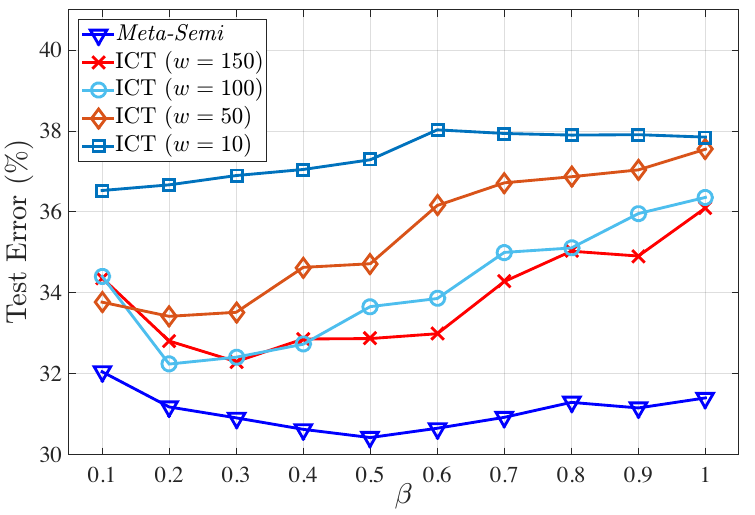}	
        \vskip -0.12in
        \caption{
        Test errors with varying $\beta$ on CIFAR-100 using 10,000 labels. The CNN-13 network is used. We also report the results of ICT \cite{verma2019interpolation} when the unsupervised consistency coefficient $w$ changes among the recommended range.
    }\label{fig:sense_1}
    \vskip -0.2in  
    \end{minipage}
    \vskip -0.25in
\end{figure}

\begin{table}
    \caption{Performance of \textit{Meta-Semi} v.s. baselines with fixed amount of training time. We report the mean test errors of both networks on CIFAR-100 with 10,000 labels. The best results are \textbf{bold-faced}}
    \label{tab:eff}
    \centering
    \subtable[CNN-13]{
        \setlength{\tabcolsep}{1.25mm}{
            \renewcommand\arraystretch{1.2}
            \resizebox{0.46\columnwidth}{!}{
            \begin{tabular}{l|c|c|c|c}
            \hline
            Training Time & 5.0h & 7.5h  & 10.0h & 12.6h  \\
            \hline
            ICT \cite{verma2019interpolation} &33.43\%&32.84\%&32.61\%&32.24\% \\
            \textit{Meta-Semi} &\textbf{32.73\%}&\textbf{31.81\%}&\textbf{31.06\%}&\textbf{30.84\%} \\
            \hline
            \end{tabular}}}
    }
    \hspace{0.02in}
    \subtable[WRN-28]{        
        \setlength{\tabcolsep}{1.25mm}{
            \renewcommand\arraystretch{1.2}
            \resizebox{0.46\columnwidth}{!}{
            \begin{tabular}{l|c|c|c|c}
            \hline
            Training Time & 13.7h & 18.3h  & 22.8h & 29.2h  \\
            \hline
            MixMatch \cite{berthelot2019mixmatch} &32.94\%&31.91\%&31.26\%&30.84\% \\
            \textit{Meta-Semi} &\textbf{31.74\%}&\textbf{30.85\%}&\textbf{30.50\%}&\textbf{30.13\%} \\
            \hline
            \end{tabular}}}
    }
    \vskip -0.25in
\end{table}

\vspace{-2ex}
\subsection{Hyper-parameter Sensitivity}
\vspace{-2ex}
\label{sec:sense}
The $\beta$ parameter for the Beta distribution in MixUp augmentation is the only additional hyper-parameter that needs to be tuned when \textit{Meta-Semi} is implemented in new SSL tasks. To study the sensitivity of our method to $\beta$, we vary the value of $\beta$, and present the test errors in Figure \ref{fig:sense_1}. For comparison, we also present the results of ICT \cite{verma2019interpolation} when its two additional hyper-parameters ($\beta$ and the unsupervised regularization coefficient $w$) change among the recommended candidates provided by the original paper. One can observe that the performance of \textit{Meta-Semi} is relatively stable when $\beta$ ranges from $0.1$ to $1$. In contrast, ICT is sensitive to both the two hyper-parameters. It has been shown that hyper-parameter searching is difficult on realistic SSL tasks \cite{oliver2018realistic}. \textit{Meta-Semi} can be more easily applied as it requires less effort for tuning hyper-parameters.

\begin{wraptable}{r}{7.5cm}
    \centering
    \begin{footnotesize}
    \vskip -0.35in
    \caption{Ablation study results. We report the test errors on CIFAR-100 with 4,000 and 10,000 labels. The CNN-13 network is used.}
    \vskip -0.1in
    \label{tab:ablation}
    \setlength{\tabcolsep}{1.25mm}{
    \renewcommand\arraystretch{1.25}
    \resizebox{7.5cm}{!}{
    \begin{tabular}{l|cc}
    \hline
    \multirow{2}{*}{Ablation}  & CIFAR-100&  CIFAR-100    \\
     & 4000 labels &  10000  labels  \\
     \hline
     Without parameter EMA  &47.68 $\pm$ 0.27\%& 37.15 $\pm$ 1.02\% \\
     One-hot pseudo labels  & 41.52 $\pm$ 0.51\% & 32.78 $\pm$ 0.41\% \\
     \hline
     MixUp on unlabeled data only  & 37.69 $\pm$ 0.50\%& 30.56 $\pm$ 0.39\% \\
     MixUp on labeled data only  &45.90 $\pm$ 0.15\%& 36.11 $\pm$ 0.21\% \\
     Without MixUp &46.71 $\pm$ 0.05\%& 35.98 $\pm$ 0.69\%  \\
     \hline
     Reweighting with the constant 1 &40.26 $\pm$ 0.64\%& 32.17 $\pm$ 0.14\%  \\
     Reweighting with -1 and 1 &45.41 $\pm$ 0.38\%& 36.39 $\pm$ 0.44\%  \\
     \hline
     \textit{Meta-Semi}  &37.61 $\pm$ 0.56\%&30.51 $\pm$ 0.32\%\\
     \textit{Meta-Semi} + ICT  & 37.12 $\pm$ 0.59\%  & 29.68 $\pm$ 0.05\% \\
    \hline
    \end{tabular}}}
    \end{footnotesize}
    \vskip -0.3in
\end{wraptable}
\vspace{-2ex}
\subsection{Efficiency of \textit{Meta-Semi}}
\vspace{-2ex}
Our method generally requires more training time for each iteration as it includes bi-level optimization. However, we find that our algorithm converges fast and if we consider a fixed amount of training time, it still outperforms the others, as shown in Table \ref{tab:eff}. 

\vspace{-2ex}
\subsection{Ablation Study}
\vspace{-2ex}
To provide additional insights into our method, we further conduct the ablation experiments by removing or altering the components of \textit{Meta-Semi}. The results are shown in Table \ref{tab:ablation}. It can be seen that parameter EMA and performing MixUp on unlabeled data are both important techniques to achieve high generalization performance. The observation is consistent with \cite{verma2019interpolation}. In addition, if all pseudo-labeled samples are weighted by the constant 1, \textit{Meta-Semi} is equivalent to a consistency based algorithm, which also shows effective performance.

\vspace{-2ex}
\section{Conclusion}
\vspace{-2ex}
In this paper, we have presented a novel semi-supervised classification algorithm under the meta-learning paradigm. The proposed \textit{Meta-Semi} algorithm is capable of adapting to various SSL tasks with impressive performance via tuning only one additional hyper-parameter, and empirically we have observed that the model performance is robust to different settings of this hyper-parameter. Theoretically, we have provided the convergence analysis to show that \textit{Meta-Semi} always converges to a stationary point under mild conditions. On four competitive datasets, \textit{Meta-Semi} has achieved state-of-the-art performance compared to existing deep SSL algorithms.

\section*{Broader Impact}
Semi-supervised learning is a widely used learning paradigm to reduce the time or economic cost of collecting annotations for large scale training sets. In this paper, we propose a \textit{Meta-Semi} algorithm that requires tuning only one addition hyper-parameter to adapt to a wide variety of semi-supervised scenarios. Our method may benefit various realistic semi-supervised applications in terms of both reducing the computational cost of hyper-parameter searching and further improving the performance of machine learning systems. For examples, search engines, social media companies and online advertising agencies all have the requirements of deploying high performance image recognition models. They can collect a large number of unannotated training samples through the Internet, annotate only a small subset of them, and implement our algorithm to obtain a highly generalized deep network rapidly, which may significantly save the cost. In addition, our algorithm may have larger impacts on medical applications, where accurate annotations usually require to be given by experts and are thus especially difficult to acquire.

For the research community, the proposed \textit{Meta-Semi} algorithm may open up the research investigating other methods to weight pseudo-labeled samples in semi-supervised learning, which is still an under explored topic. 

On the other hand, since the proposed algorithm is mainly based on convolutional networks (CNNs), it may suffer from the common problems of CNNs, such as vulnerable to adversarial attacks. Moreover, semi-supervised learning may have privacy risks. Since the companies need to collect a large amount of unannotated data for semi-supervised learning algorithms, they may potentially infringe privacy by improperly accessing user data.

In general, we believe that the potential positive impacts of this paper significantly outweigh the negative ones in terms of both the practical implementations and the research value.

\section*{Acknowledgments}
This work is supported in part by the Ministry of Science and Technology of China under Grant 2018AAA0101604, the National Natural Science Foundation of China under Grants 61906106 and 61936009, the Institute for Guo Qiang of Tsinghua University and Beijing Academy of Artificial Intelligence.
In particular, we appreciate the valuable discussion with Yitong Xia and Hong Zhang.

\bibliographystyle{plain}
\bibliography{neurips_2020}

\begin{thebibliography}{10}

\bibitem{andrychowicz2016learning}
Marcin Andrychowicz, Misha Denil, Sergio Gomez, Matthew~W Hoffman, David Pfau,
  Tom Schaul, Brendan Shillingford, and Nando De~Freitas.
\newblock Learning to learn by gradient descent by gradient descent.
\newblock In {\em NeurIPS}, pages 3981--3989, 2016.

\bibitem{athiwaratkun2018there}
Ben Athiwaratkun, Marc Finzi, Pavel Izmailov, and Andrew~Gordon Wilson.
\newblock There are many consistent explanations of unlabeled data: Why you
  should average.
\newblock 2019.

\bibitem{bachman2014learning}
Philip Bachman, Ouais Alsharif, and Doina Precup.
\newblock Learning with pseudo-ensembles.
\newblock In {\em NeurIPS}, pages 3365--3373, 2014.

\bibitem{bengio200611}
Yoshua Bengio, Olivier Delalleau, and Nicolas Le~Roux.
\newblock 11 label propagation and quadratic criterion.
\newblock 2006.

\bibitem{bergstra2012random}
James Bergstra and Yoshua Bengio.
\newblock Random search for hyper-parameter optimization.
\newblock {\em Journal of machine learning research}, 13(Feb):281--305, 2012.

\bibitem{berthelot2019remixmatch}
David Berthelot, Nicholas Carlini, Ekin~D Cubuk, Alex Kurakin, Kihyuk Sohn, Han
  Zhang, and Colin Raffel.
\newblock Remixmatch: Semi-supervised learning with distribution alignment and
  augmentation anchoring.
\newblock In {\em ICLR}, 2020.

\bibitem{berthelot2019mixmatch}
David Berthelot, Nicholas Carlini, Ian Goodfellow, Nicolas Papernot, Avital
  Oliver, and Colin Raffel.
\newblock Mixmatch: A holistic approach to semi-supervised learning.
\newblock In {\em NeurIPS}, 2019.

\bibitem{bertsekas1997nonlinear}
Dimitri~P Bertsekas.
\newblock Nonlinear programming.
\newblock {\em Athena Scientific}, 1997.

\bibitem{SSL_book}
O.~Chapelle, B.~Sch{\"o}lkopf, and A.~Zien.
\newblock {\em Semi-Supervised Learning}.
\newblock Adaptive computation and machine learning. MIT Press, Cambridge, MA,
  USA, September 2006.

\bibitem{coates2011analysis}
Adam Coates, Andrew Ng, and Honglak Lee.
\newblock An analysis of single-layer networks in unsupervised feature
  learning.
\newblock In {\em AISTATS}, pages 215--223, 2011.

\bibitem{denton2016semi}
Emily Denton, Sam Gross, and Rob Fergus.
\newblock Semi-supervised learning with context-conditional generative
  adversarial networks.
\newblock {\em arXiv preprint arXiv:1611.06430}, 2016.

\bibitem{finn2017model}
Chelsea Finn, Pieter Abbeel, and Sergey Levine.
\newblock Model-agnostic meta-learning for fast adaptation of deep networks.
\newblock In {\em ICML}, pages 1126--1135. JMLR. org, 2017.

\bibitem{grandvalet2005semi}
Yves Grandvalet and Yoshua Bengio.
\newblock Semi-supervised learning by entropy minimization.
\newblock In {\em NeurIPS}, pages 529--536, 2005.

\bibitem{He_2016_CVPR}
Kaiming He, Xiangyu Zhang, Shaoqing Ren, and Jian Sun.
\newblock Deep residual learning for image recognition.
\newblock In {\em CVPR}, pages 770--778, 2016.

\bibitem{huang2017snapshot}
Gao Huang, Yixuan Li, Geoff Pleiss, Zhuang Liu, John~E Hopcroft, and Kilian~Q
  Weinberger.
\newblock Snapshot ensembles: Train 1, get m for free.
\newblock In {\em ICLR}, 2017.

\bibitem{huang2019convolutional}
Gao Huang, Zhuang Liu, Geoff Pleiss, Laurens Van Der~Maaten, and Kilian
  Weinberger.
\newblock Convolutional networks with dense connectivity.
\newblock {\em IEEE transactions on pattern analysis and machine intelligence},
  2019.

\bibitem{joachims1999transductive}
Thorsten Joachims.
\newblock Transductive inference for text classification using support vector
  machines.
\newblock In {\em ICML}, volume~99, pages 200--209, 1999.

\bibitem{joachims2003transductive}
Thorsten Joachims.
\newblock Transductive learning via spectral graph partitioning.
\newblock In {\em AAAI}, pages 290--297, 2003.

\bibitem{kingma2014semi}
Durk~P Kingma, Shakir Mohamed, Danilo~Jimenez Rezende, and Max Welling.
\newblock Semi-supervised learning with deep generative models.
\newblock In {\em Advances in neural information processing systems}, pages
  3581--3589, 2014.

\bibitem{krizhevsky2009learning}
Alex Krizhevsky, Geoffrey Hinton, et~al.
\newblock Learning multiple layers of features from tiny images.
\newblock Technical report, Citeseer, 2009.

\bibitem{krizhevsky2012imagenet}
Alex Krizhevsky, Ilya Sutskever, and Geoffrey~E Hinton.
\newblock Imagenet classification with deep convolutional neural networks.
\newblock In {\em NeurIPS}, pages 1097--1105, 2012.

\bibitem{laine2016temporal}
Samuli Laine and Timo Aila.
\newblock Temporal ensembling for semi-supervised learning.
\newblock {\em arXiv preprint arXiv:1610.02242}, 2016.

\bibitem{lake2017building}
Brenden~M Lake, Tomer~D Ullman, Joshua~B Tenenbaum, and Samuel~J Gershman.
\newblock Building machines that learn and think like people.
\newblock {\em Behavioral and brain sciences}, 40, 2017.

\bibitem{lecun2015deep}
Yann LeCun, Yoshua Bengio, and Geoffrey Hinton.
\newblock Deep learning.
\newblock {\em nature}, 521(7553):436, 2015.

\bibitem{lee2013pseudo}
Dong-Hyun Lee.
\newblock Pseudo-label: The simple and efficient semi-supervised learning
  method for deep neural networks.
\newblock In {\em ICML Workshop on Challenges in Representation Learning},
  volume~3, page~2, 2013.

\bibitem{loshchilov2016sgdr}
Ilya Loshchilov and Frank Hutter.
\newblock Sgdr: Stochastic gradient descent with warm restarts.
\newblock {\em arXiv preprint arXiv:1608.03983}, 2016.

\bibitem{luo2018smooth}
Yucen Luo, Jun Zhu, Mengxi Li, Yong Ren, and Bo~Zhang.
\newblock Smooth neighbors on teacher graphs for semi-supervised learning.
\newblock In {\em CVPR}, pages 8896--8905, 2018.

\bibitem{miyato2018virtual}
Takeru Miyato, Shin-ichi Maeda, Masanori Koyama, and Shin Ishii.
\newblock Virtual adversarial training: a regularization method for supervised
  and semi-supervised learning.
\newblock {\em IEEE transactions on pattern analysis and machine intelligence},
  41(8):1979--1993, 2018.

\bibitem{netzer2011reading}
Yuval Netzer, Tao Wang, Adam Coates, Alessandro Bissacco, Bo~Wu, and Andrew~Y
  Ng.
\newblock Reading digits in natural images with unsupervised feature learning.
\newblock In {\em NuerIPS Workshop on Deep Learning and Unsupervised Feature
  Learning}, 2011.

\bibitem{odena2016semi}
Augustus Odena.
\newblock Semi-supervised learning with generative adversarial networks.
\newblock {\em arXiv preprint arXiv:1606.01583}, 2016.

\bibitem{oliver2018realistic}
Avital Oliver, Augustus Odena, Colin~A Raffel, Ekin~Dogus Cubuk, and Ian
  Goodfellow.
\newblock Realistic evaluation of deep semi-supervised learning algorithms.
\newblock In {\em NeurIPS}, pages 3235--3246, 2018.

\bibitem{park2018adversarial}
Sungrae Park, JunKeon Park, Su-Jin Shin, and Il-Chul Moon.
\newblock Adversarial dropout for supervised and semi-supervised learning.
\newblock In {\em AAAI}, 2018.

\bibitem{ravi2016optimization}
Sachin Ravi and Hugo Larochelle.
\newblock Optimization as a model for few-shot learning.
\newblock In {\em ICLR}, 2017.

\bibitem{reddi2016stochastic}
Sashank~J Reddi, Ahmed Hefny, Suvrit Sra, Barnabas Poczos, and Alex Smola.
\newblock Stochastic variance reduction for nonconvex optimization.
\newblock In {\em ICML}, pages 314--323, 2016.

\bibitem{ren2018meta}
Mengye Ren, Eleni Triantafillou, Sachin Ravi, Jake Snell, Kevin Swersky,
  Joshua~B Tenenbaum, Hugo Larochelle, and Richard~S Zemel.
\newblock Meta-learning for semi-supervised few-shot classification.
\newblock In {\em ICLR}, 2018.

\bibitem{ren2018learning}
Mengye Ren, Wenyuan Zeng, Bin Yang, and Raquel Urtasun.
\newblock Learning to reweight examples for robust deep learning.
\newblock In {\em ICML}, 2018.

\bibitem{sajjadi2016regularization}
Mehdi Sajjadi, Mehran Javanmardi, and Tolga Tasdizen.
\newblock Regularization with stochastic transformations and perturbations for
  deep semi-supervised learning.
\newblock In {\em NeurIPS}, pages 1163--1171, 2016.

\bibitem{simonyan2014very}
Karen Simonyan and Andrew Zisserman.
\newblock Very deep convolutional networks for large-scale image recognition.
\newblock {\em arXiv preprint arXiv:1409.1556}, 2014.

\bibitem{szegedy2015going}
Christian Szegedy, Wei Liu, Yangqing Jia, Pierre Sermanet, Scott Reed, Dragomir
  Anguelov, Dumitru Erhan, Vincent Vanhoucke, and Andrew Rabinovich.
\newblock Going deeper with convolutions.
\newblock In {\em CVPR}, pages 1--9, 2015.

\bibitem{tarvainen2017mean}
Antti Tarvainen and Harri Valpola.
\newblock Mean teachers are better role models: Weight-averaged consistency
  targets improve semi-supervised deep learning results.
\newblock In {\em NeurIPS}, pages 1195--1204, 2017.

\bibitem{turian2010word}
Joseph Turian, Lev Ratinov, and Yoshua Bengio.
\newblock Word representations: a simple and general method for semi-supervised
  learning.
\newblock In {\em ACL}, pages 384--394. Association for Computational
  Linguistics, 2010.

\bibitem{verma2019interpolation}
Vikas Verma, Alex Lamb, Juho Kannala, Yoshua Bengio, and David Lopez-Paz.
\newblock Interpolation consistency training for semi-supervised learning.
\newblock In {\em IJCAI}, 2019.

\bibitem{zhang2017mixup}
Hongyi Zhang, Moustapha Cisse, Yann~N Dauphin, and David Lopez-Paz.
\newblock Mixup: Beyond empirical risk minimization.
\newblock In {\em ICLR}, 2018.

\bibitem{zhao2015stacked}
Junbo Zhao, Michael Mathieu, Ross Goroshin, and Yann Lecun.
\newblock Stacked what-where auto-encoders.
\newblock {\em arXiv preprint arXiv:1506.02351}, 2015.

\bibitem{zhu2003semi}
Xiaojin Zhu, Zoubin Ghahramani, and John~D Lafferty.
\newblock Semi-supervised learning using gaussian fields and harmonic
  functions.
\newblock In {\em ICML}, pages 912--919, 2003.

\end{thebibliography}


\newpage
\appendix
\section*{Appendix: Meta-Semi: A Meta-learning Approach for Semi-supervised Learning}

\section{Proof of Proposition 1}
This section provides the proof of Proposition 1.
\begin{proposition_1}
    Suppose that $\overline{\bm{\theta}}^t_M$ is given by $M$ times of gradient descents starting from $\overline{\bm{\theta}}^t_0=\bm{\theta}^{t}$. Then we have
    \begin{equation}
           \left. \frac{\partial\sum_{i=1}^{|\tilde{\mathcal{X}}|}L(\tilde{\bm{y}}_i, p(\tilde{\bm{x}}_i|\overline{\bm{\theta}}^t_M))}{\partial w_j}\right|_{\bm{w}=0} 
        =  \left. M \left[  \frac{\partial\sum_{i=1}^{|\tilde{\mathcal{X}}|}L(\tilde{\bm{y}}_i, p(\tilde{\bm{x}}_i|\overline{\bm{\theta}}^t_1))}{\partial w_j}\right|_{\bm{w}=0}  \right]\!\!,\ \forall\  1\!\leq\!j\!\leq |\tilde{\mathcal{U}}|.
    \end{equation}
\end{proposition_1}
\begin{proof}
    According to the updating rule $ \overline{\bm{\theta}}^t_{M} = \overline{\bm{\theta}}^t_{M-1}\!-\!\alpha^t \nabla_{\overline{\bm{\theta}}^t_{M-1}}\!\!\sum_{k=1}^{|\tilde{\mathcal{U}}|} w_k L(\hat{\bm{y}}_k, p(\tilde{\bm{u}}_k|\overline{\bm{\theta}}^t_{M-1}))$, we obtain:
    \begin{align}
            \label{eq_2}
          &  \left. \frac{\partial\sum_{i=1}^{|\tilde{\mathcal{X}}|}L(\tilde{\bm{y}}_i, p(\tilde{\bm{x}}_i|\overline{\bm{\theta}}^t_M))}{\partial w_j}\right|_{\bm{w}=0} \\
        \!= & \! \left[\!\left.\frac{\partial \!\sum_{i=1}^{|\tilde{\mathcal{X}}|} \!L(\tilde{\bm{y}}_i, p(\tilde{\bm{x}}_i|\overline{\bm{\theta}}^t_M))}{\partial \overline{\bm{\theta}}^t_M} \! \right]^{\!T} \! \! \left[ \!\frac{\partial \overline{\bm{\theta}}^t_{M}}{ \partial w_j} \! \right] \right|_{\bm{w}=0} \\
        \!= &\! \left[\!\left.\frac{\partial \!\sum_{i=1}^{|\tilde{\mathcal{X}}|} \!L(\tilde{\bm{y}}_i, p(\tilde{\bm{x}}_i|\overline{\bm{\theta}}^t_M))}{\partial \overline{\bm{\theta}}^t_M} \! \right]^{\!T} \! \!
        \left[ \!\frac{\partial (\overline{\bm{\theta}}^t_{M-1}\!-\!\alpha^t \nabla_{\overline{\bm{\theta}}^t_{M-1}} 
        \!\!\sum_{k=1}^{|\tilde{\mathcal{U}}|} w_k L(\hat{\bm{y}}_k, p(\tilde{\bm{u}}_k|\overline{\bm{\theta}}^t_{M-1}))) }{ \partial w_j} \! \right] \right|_{\bm{w}=0} \\ 
        \!= & \!\left[\!\frac{\partial \!\sum_{i=1}^{|\tilde{\mathcal{X}}|} \!L(\tilde{\bm{y}}_i, p(\tilde{\bm{x}}_i|\overline{\bm{\theta}}^t_M))}{\partial \overline{\bm{\theta}}^t_M} \! \right]^{\!T} \! \!\left[ \frac{\partial  \overline{\bm{\theta}}^t_{M-1}}{\partial w_j} - \alpha^t \sum_{k=1}^{|\tilde{\mathcal{U}}|} \left[
          \frac{\partial w_k\!\nabla_{\overline{\bm{\theta}}^t_{M-1}}\!\!L(\hat{\bm{y}}_k, p(\tilde{\bm{u}}_k|\overline{\bm{\theta}}^t_{M-1}))}{\partial w_k}\ 
          \frac{\partial w_k}{\partial w_j} \right. \right.
          \nonumber\\
           &\  \ \ \ \ \ \ \ \ \ \ \ \ \ \ \ \ \ \ \  \ \ \ \ \ \ \ \ 
           + \left.\left. \left. \frac{\partial w_k  \nabla_{\overline{\bm{\theta}}^t_{M-1}}\!\!L(\hat{\bm{y}}_k, p(\tilde{\bm{u}}_k|\overline{\bm{\theta}}^t_{M-1}))}{\partial \nabla_{\overline{\bm{\theta}}^t_{M-1}}\!\!L(\hat{\bm{y}}_k, p(\tilde{\bm{u}}_k|\overline{\bm{\theta}}^t_{M-1}))}\  \frac{\partial \nabla_{\overline{\bm{\theta}}^t_{M-1}}\!\!L(\hat{\bm{y}}_k, p(\tilde{\bm{u}}_k|\overline{\bm{\theta}}^t_{M-1}))}{\partial w_j}  \right]
          \right]  \right|_{\bm{w}=0} \\
        \!= & \!\left[\!\left.\frac{\partial \!\sum_{i=1}^{|\tilde{\mathcal{X}}|} \!L(\tilde{\bm{y}}_i, p(\tilde{\bm{x}}_i|\overline{\bm{\theta}}^t_M))}{\partial \overline{\bm{\theta}}^t_M} \! \right]^{\!T} \! \!\left[ \frac{\partial  \overline{\bm{\theta}}^t_{M-1}}{\partial w_j}-\alpha^t \! \left[ \frac{\partial L(\hat{\bm{y}}_j, p(\tilde{\bm{u}}_j|\overline{\bm{\theta}}^t_{M-1}))}{\partial \overline{\bm{\theta}}^t_{M-1}}\!  \right. \right. \right.\nonumber\\
        & \  \ \ \ \ \ \ \ \ \ \ \ \ \ \ \ \ \ \ \  \ \ \ \ \ \ \  \  \ \ \ \ \ \  \ \ \ \ \ \ \  \ \ \ \ \ \ \ \ \ \ \ \ \ \ \ \ \ \ \  \ \ \ \ \ \
        \left.\left.\left. +\!\sum_{k=1}^{|\tilde{\mathcal{U}}|}
        \! w_k\! \frac{\partial \nabla_{\overline{\bm{\theta}}^t_{M-1}}\!\!L(\hat{\bm{y}}_k, p(\tilde{\bm{u}}_k|\overline{\bm{\theta}}^t_{M-1}))}{\partial w_j} \right] \right] \right|_{\bm{w}=0} \\
        \label{eq_7}
        \!= & \!\left[\!\left.\frac{\partial \!\sum_{i=1}^{|\tilde{\mathcal{X}}|} \!L(\tilde{\bm{y}}_i, p(\tilde{\bm{x}}_i|{\bm{\theta}}^t))}{\partial {\bm{\theta}}^t} \! \right]^{\!T} \! \!\left[ \frac{\partial  \overline{\bm{\theta}}^t_{M-1}}{\partial w_j}\right|_{\bm{w}=0} - \alpha^t  \frac{\partial L(\hat{\bm{y}}_j, p(\tilde{\bm{u}}_j|{\bm{\theta}}^t))}{\partial {\bm{\theta}}^t}  \right] \\
        \!= & \!\left[\!\left.\frac{\partial \!\sum_{i=1}^{|\tilde{\mathcal{X}}|} \!L(\tilde{\bm{y}}_i, p(\tilde{\bm{x}}_i|{\bm{\theta}}^t))}{\partial {\bm{\theta}}^t} \! \right]^{\!T} \! \!\left[ \frac{\partial  \overline{\bm{\theta}}^t_{M-2}}{\partial w_j}\right|_{\bm{w}=0} - 2\alpha^t  \frac{\partial L(\hat{\bm{y}}_j, p(\tilde{\bm{u}}_j|{\bm{\theta}}^t))}{\partial {\bm{\theta}}^t}  \right] \\
        \label{eq_9}
        \!= & \!\left[\!\left.\frac{\partial \!\sum_{i=1}^{|\tilde{\mathcal{X}}|} \!L(\tilde{\bm{y}}_i, p(\tilde{\bm{x}}_i|{\bm{\theta}}^t))}{\partial {\bm{\theta}}^t} \! \right]^{\!T} \! \!\left[ \frac{\partial  \overline{\bm{\theta}}^t_{0}}{\partial w_j}\right|_{\bm{w}=0} - M\alpha^t  \frac{\partial L(\hat{\bm{y}}_j, p(\tilde{\bm{u}}_j|{\bm{\theta}}^t))}{\partial {\bm{\theta}}^t}  \right] \\
        \label{eq_10}
        \!= & - M\alpha^t \!\left[\!\frac{\partial \!\sum_{i=1}^{|\tilde{\mathcal{X}}|} \!L(\tilde{\bm{y}}_i, p(\tilde{\bm{x}}_i|{\bm{\theta}}^t))}{\partial {\bm{\theta}}^t} \! \right]^{\!T} \! \!\left[   \frac{\partial L(\hat{\bm{y}}_j, p(\tilde{\bm{u}}_j|{\bm{\theta}}^t))}{\partial {\bm{\theta}}^t}  \right].
    \end{align}
    In the above, the Eq. (\ref{eq_7}) is obtained as we have $\overline{\bm{\theta}}^t_{M} = \overline{\bm{\theta}}^t_{M-1} = \ldots =\overline{\bm{\theta}}^t_{0} $ when $\bm{w} = 0$. The Eq. (\ref{eq_9}) follows repeatedly using Eqs. (\ref{eq_2}-\ref{eq_7}). Let $M=1$, we have
    \begin{equation}
        \label{eq_11}
        \left. \frac{\partial\sum_{i=1}^{|\tilde{\mathcal{X}}|}L(\tilde{\bm{y}}_i, p(\tilde{\bm{x}}_i|\overline{\bm{\theta}}^t_1))}{\partial w_j}\right|_{\bm{w}=0} 
        = - \alpha^t \!\left[\!\frac{\partial \!\sum_{i=1}^{|\tilde{\mathcal{X}}|} \!L(\tilde{\bm{y}}_i, p(\tilde{\bm{x}}_i|{\bm{\theta}}^t))}{\partial {\bm{\theta}}^t} \! \right]^{\!T} \! \!\left[   \frac{\partial L(\hat{\bm{y}}_j, p(\tilde{\bm{u}}_j|{\bm{\theta}}^t))}{\partial {\bm{\theta}}^t}  \right].
    \end{equation}
    By combining Eq. (\ref{eq_10}) and Eq. (\ref{eq_11}), we prove the desired proposition.
\end{proof}

\section{Proof of Proposition 2}
This Section provides the proof of Proposition 2.
In our proof, the MixUp augmentation is considered since it is an important part of our algorithm.
We begin with a Lemma \cite{bertsekas1997nonlinear} based on the definition of Lipschitz-smooth.
\begin{definition_1}
    A continuously differentiable function $f: \mathbb{R}^n \to \mathbb{R}$ is said to be Lipschitz-smooth with constant $L$ if
\begin{align*}
\lVert \nabla f(x) - \nabla f(y) \rVert \leq L \lVert x - y \rVert, \forall x, y \in \mathbb{R}^n.
\end{align*}
\end{definition_1}
\begin{lemma}
    \label{lemma}
    Assume that the continuously differentiable function $f: \mathbb{R}^n \to \mathbb{R}$ is Lipschitz-smooth with the scalar $L$. Then
    \begin{align*}
        f(x + y) \leq f(x) + y^T \nabla f(x) + \frac{L}{2} \lVert y \lVert^2, \forall x, y \in \mathbb{R}^n.
    \end{align*}
\end{lemma}
Then we introduce a mild assumption to restrict the expected norm of the gradients. The assumption is empirically shown to be generally held in semi-supervised learning. 

To clearly present the assumption and the proof, we first define two new symbols.
Suppose that the supervised loss on the labeled mini-batch $\tilde{\mathcal{X}}$ at $t^{\text{th}}$ step is denoted by 
\begin{equation}
    G(\tilde{\mathcal{X}}, \bm{\theta}^t) = \sum_{i=1}^{|\tilde{\mathcal{X}}|} L(\tilde{\bm{y}}_i, p(\tilde{\bm{x}}_i|\bm{\theta}^t)).
\end{equation}
In addition, we denote the dynamically weighted loss of pseudo-labeled samples by
\begin{equation}
    F(\tilde{\mathcal{X}}, \tilde{\mathcal{U}}, \bm{\theta}^t)= \mathcal{L}_{meta} =\frac{1}{\sum_{j=1}^{|\tilde{\mathcal{U}}|}w_j^t} \sum_{j=1}^{|\tilde{\mathcal{U}}|} w_j^t L(\hat{\bm{y}}_j, p(\tilde{\bm{u}}_j|\bm{\theta}^t)),
\end{equation}
Note that we assume $F(\tilde{\mathcal{X}}, \tilde{\mathcal{U}}, \bm{\theta}^t)=0$ if $\sum_{j=1}^{|\tilde{\mathcal{U}}|}w_j^t=0$.
Then we have the following assumption.
\begin{assumption_1}
    \label{assum:bound_1}
    For all $t \geq 0$, there exists a positive scalar $\sigma$, such that
    \begin{align*}
        \mathbb{E}_{\tilde{\mathcal{X}}, \tilde{\mathcal{U}}} \lVert  \nabla_{\!\!\bm{\theta}^t} F(\tilde{\mathcal{X}}, \tilde{\mathcal{U}}, \bm{\theta}^t) \lVert^2 \leq \sigma\lVert \nabla_{\!\!\bm{\theta}^t} \mathbb{E}_{\tilde{\mathcal{X}}} G(\tilde{\mathcal{X}}, \bm{\theta}^t) \lVert^2.
    \end{align*}
    \vskip -0.3in
\end{assumption_1}
Now we are ready to present the detailed proof. 
Our proof is partially inspired by the proof of convergence for gradient based methods with diminishing stepsize provided by \cite{bertsekas1997nonlinear}.
\begin{proposition_2}
    Assume that the loss function on labeled data $G(\tilde{\mathcal{X}}, \bm{\theta}^t)$ is Lipschitz-smooth with regards to $\bm{\theta}^t$ for all $\tilde{\mathcal{X}}$, and that assumption \ref{assum:bound} holds. Suppose also that the learning rate $\alpha^t > 0$ satisfies:
    \begin{equation}
        \lim_{t \to \infty} \alpha^t = 0, \ \ \ \ \sum_{t=0}^{\infty} \alpha^t = \infty.
    \end{equation}
    Then every limit point of the sequence $\{ \bm{\theta}^t \}$ generated by Meta-Semi is a stationary point of $\mathbb{E}_{\tilde{\mathcal{X}}} G(\tilde{\mathcal{X}}, \bm{\theta}^t)$, namely, 
    \begin{align*}
        \lim_{t \to \infty} \lVert \nabla_{\!\!\bm{\theta}^t} \mathbb{E}_{\tilde{\mathcal{X}}} G(\tilde{\mathcal{X}}, \bm{\theta}^t) \lVert = 0.
    \end{align*}
\end{proposition_2}
\begin{proof}
    The MixUp data augmentation needs to be considered because it is leveraged to generate $\tilde{\mathcal{X}}$ and $\tilde{\mathcal{U}}$ from the original data. On the basis of the original labeled samples $\mathcal{X}$ and unlabeled samples $\mathcal{U}$ (associated with original pseudo labels), we have
    \begin{equation}
        \tilde{\mathcal{X}} = \text{MixUp}(\mathcal{X}, \text{Shuffle}(\mathcal{X}), \lambda_1), \lambda_1\!\sim\!\text{Beta}(\alpha, \alpha),
    \end{equation}
    \begin{equation}
        \mathcal{W} = \text{Concat}(\mathcal{X}, \mathcal{U}),
    \end{equation}
    \begin{equation}
        \tilde{\mathcal{U}} = \text{MixUp}(\mathcal{W}, \text{Shuffle}(\mathcal{W}), \lambda_2), \lambda_2\!\sim\!\text{Beta}(\alpha, \alpha).
    \end{equation}
    Given that the MixUp augmentation is performed between the mini-batch and itself with certain random permutation, we define the expected loss over all possible permutations by
    \begin{equation}
        \begin{split}
            \overline{G}(\mathcal{X}, \bm{\theta}^t, \lambda_1) = & \mathop{\mathbb{E}}_{
            \tilde{\mathcal{X}} \in
            \text{MixUp}(\mathcal{X}, \text{Shuffle}(\mathcal{X}), \lambda_1)
        }G(\tilde{\mathcal{X}}, \bm{\theta}^t)\\
        = & \mathop{\mathbb{E}}_{
            \tilde{\mathcal{X}} \in
            \text{MixUp}(\mathcal{X}, \text{Shuffle}(\mathcal{X}), \lambda_1)
        } \sum_{i=1}^{|\tilde{\mathcal{X}}|} L(\tilde{\bm{y}}_i, p(\tilde{\bm{x}}_i|\bm{\theta}^t)),
        \end{split}
    \end{equation}
    \begin{equation}
        \begin{split}
            \overline{F}(\tilde{\mathcal{X}}, \mathcal{X}, \mathcal{U}, \bm{\theta}^t, \lambda_2) = &  \mathop{\mathbb{E}}_{
            \tilde{\mathcal{U}} \in
            \text{MixUp}(\mathcal{W}, \text{Shuffle}(\mathcal{W}), \lambda_2)
        } \left[\sum_{j=1}^{|\tilde{\mathcal{U}}|}w_j^t \right]\!F(\tilde{\mathcal{X}}, \tilde{\mathcal{U}}, \bm{\theta}^t)\\
        = & \mathop{\mathbb{E}}_{
            \tilde{\mathcal{U}} \in
            \text{MixUp}(\mathcal{W}, \text{Shuffle}(\mathcal{W}), \lambda_2)
        }\sum_{j=1}^{|\tilde{\mathcal{U}}|} w_j^t L(\hat{\bm{y}}_j, p(\tilde{\bm{u}}_j|\bm{\theta}^t)),
        \end{split}
    \end{equation}
    where the first argument $\tilde{\mathcal{X}}$ of $\overline{F}(\cdot)$ is used for determining the dynamic weights of pseudo-labeled samples. Then we solve $\overline{G}$ and $\overline{F}$ in a closed form. Consider the following problem: $N$ different items are paired to the same $N$ items. Obviously, there are $N!$ modes of pairing in total. If we fix certain pair, we will have $(N-1)!$ modes of pairing left. Therefore, if we combine all $N!$ possible pairing modes together, we will find that any item is paired to every item (including itself) for $(N-1)!$ times. Similarly, in our problem, it is easy to obtain
    \begin{equation}
        \overline{G}(\mathcal{X}, \bm{\theta}^t, \lambda_1) =  \frac{1}{|\mathcal{X}|!} \sum_{\bm{x}_i, \bm{x}_j \in \mathcal{X}}
        (|\mathcal{X}|-1)! f_{ij}(\bm{\theta}^t, \lambda_1)
        =  \frac{1}{|{\tilde{\mathcal{X}}}|} \sum_{\bm{x}_i, \bm{x}_j \in  \mathcal{X}}
         f_{ij}(\bm{\theta}^t, \lambda_1),
    \end{equation}
    where $\bm{x}_i, \bm{x}_j$ are the original samples that can be either labeled or unlabeled. The loss of the augmented sample generated by performing MixUp augmentation between $\bm{x}_i$ and $\bm{x}_j$ with $\lambda_1$ is denoted by $f_{ij}(\bm{\theta}^t, \lambda_1)$. In a similar way, we can obtain
    \begin{equation}
        \overline{F}(\tilde{\mathcal{X}}, \mathcal{X}, \mathcal{U}, \bm{\theta}^t, \lambda_2)
        =  \frac{1}{|{\tilde{\mathcal{U}}}|}  \sum_{\bm{x}_i, \bm{x}_j \in \mathcal{X} \cup \mathcal{U}}
        w_{ij}^t(\lambda_2) \ f_{ij}(\bm{\theta}^t, \lambda_2),
    \end{equation}
    where $w_{ij}^t(\lambda_2)$ is the dynamic weight determined by 
    \begin{equation}
    \label{weight}
    w_{ij}^t(\lambda_2) =
    \begin{cases}
        1
        & 
        [\nabla_{\!\!\bm{\theta}^t} G(\tilde{\mathcal{X}}, \bm{\theta}^t)]^T [\nabla_{\!\!\bm{\theta}^t} f_{ij}(\bm{\theta}^t, \lambda_2)] \geq 0 \vspace{1ex} \\
        0
        & 
        [\nabla_{\!\!\bm{\theta}^t} G(\tilde{\mathcal{X}}, \bm{\theta}^t)]^T [\nabla_{\!\!\bm{\theta}^t} f_{ij}(\bm{\theta}^t, \lambda_2)] < 0
    \end{cases}.
    \end{equation}
    Now, consider the following inequation
    \begin{align}
        [\nabla_{\!\!\bm{\theta}^t} G(\tilde{\mathcal{X}}, \bm{\theta}^t)]^T [\nabla_{\!\!\bm{\theta}^t} F(\tilde{\mathcal{X}}, \tilde{\mathcal{U}}, \bm{\theta}^t)]
        =\  & \frac{1}{\sum_{j=1}^{|\tilde{\mathcal{U}}|}w_j^t} 
        [\nabla_{\!\!\bm{\theta}^t} G(\tilde{\mathcal{X}}, \bm{\theta}^t)]^T 
        \left[
            \sum_{j=1}^{|\tilde{\mathcal{U}}|} w_j^t 
        [\nabla_{\!\!\bm{\theta}^t} L(\hat{\bm{y}}_j, p(\tilde{\bm{u}}_j|\bm{\theta}^t))] 
        \right]\\
        \geq\  & \frac{1}{|\tilde{\mathcal{U}}|} [\nabla_{\!\!\bm{\theta}^t} G(\tilde{\mathcal{X}}, \bm{\theta}^t)]^T 
        \left[
            \sum_{j=1}^{|\tilde{\mathcal{U}}|} w_j^t 
        [\nabla_{\!\!\bm{\theta}^t} L(\hat{\bm{y}}_j, p(\tilde{\bm{u}}_j|\bm{\theta}^t))] 
        \right].
    \end{align}
    By taking the expectation over $\tilde{\mathcal{U}} \in \text{MixUp}(\mathcal{W}, \text{Shuffle}(\mathcal{W}), \lambda_2)$, we further obtain
    \begin{align}
        & \mathop{\mathbb{E}}_{
            \tilde{\mathcal{U}} \in
            \text{MixUp}(\mathcal{W}, \text{Shuffle}(\mathcal{W}), \lambda_2)
        }[\nabla_{\!\!\bm{\theta}^t} G(\tilde{\mathcal{X}}, \bm{\theta}^t)]^T [\nabla_{\!\!\bm{\theta}^t} F(\tilde{\mathcal{X}}, \tilde{\mathcal{U}}, \bm{\theta}^t)] \\
        \geq & \ \frac{1}{|\tilde{\mathcal{U}}|} [\nabla_{\!\!\bm{\theta}^t} G(\tilde{\mathcal{X}}, \bm{\theta}^t)]^T  [\nabla_{\!\!\bm{\theta}^t} 
        \overline{F}(\tilde{\mathcal{X}}, \mathcal{X}, \mathcal{U}, \bm{\theta}^t, \lambda_2)] \\
        = & \ \frac{1}{|\tilde{\mathcal{U}}|^2} \sum_{\bm{x}_i, \bm{x}_j \in \mathcal{X} \cup \mathcal{U}}
        w_{ij}^t(\lambda_2) \ [\nabla_{\!\!\bm{\theta}^t} G(\tilde{\mathcal{X}}, \bm{\theta}^t)]^T[\nabla_{\!\!\bm{\theta}^t} f_{ij}(\bm{\theta}^t, \lambda_2)] \\
        \geq& \ \frac{1}{|\tilde{\mathcal{U}}|^2} \sum_{\bm{x}_i, \bm{x}_j \in \mathcal{X} }
        w_{ij}^t(\lambda_2) \ [\nabla_{\!\!\bm{\theta}^t} G(\tilde{\mathcal{X}}, \bm{\theta}^t)]^T[\nabla_{\!\!\bm{\theta}^t} f_{ij}(\bm{\theta}^t, \lambda_2)] \\
        \geq& \ \frac{1}{|\tilde{\mathcal{U}}|^2} \sum_{\bm{x}_i, \bm{x}_j \in \mathcal{X} }
        [\nabla_{\!\!\bm{\theta}^t} G(\tilde{\mathcal{X}}, \bm{\theta}^t)]^T[\nabla_{\!\!\bm{\theta}^t} f_{ij}(\bm{\theta}^t, \lambda_2)] \\
        = & \ \frac{|\tilde{\mathcal{X}}|}{|\tilde{\mathcal{U}}|^2}  [\nabla_{\!\!\bm{\theta}^t} G(\tilde{\mathcal{X}}, \bm{\theta}^t)]^T 
        [\nabla_{\!\!\bm{\theta}^t} \overline{G}(\mathcal{X}, \bm{\theta}^t, \lambda_2)].
    \end{align} 
    Then by taking the expectation over $ \tilde{\mathcal{X}} \in \text{MixUp}(\mathcal{X}, \text{Shuffle}(\mathcal{X}), \lambda_1)$, we have
    \begin{align}
        & \mathop{\mathbb{E}}_{
            \tilde{\mathcal{X}} \in \text{MixUp}(\mathcal{X}, \text{Shuffle}(\mathcal{X}), \lambda_1)
        }\ 
        \mathop{\mathbb{E}}_{
            \tilde{\mathcal{U}} \in
            \text{MixUp}(\mathcal{W}, \text{Shuffle}(\mathcal{W}), \lambda_2)
        }[\nabla_{\!\!\bm{\theta}^t} G(\tilde{\mathcal{X}}, \bm{\theta}^t)]^T [\nabla_{\!\!\bm{\theta}^t} F(\tilde{\mathcal{X}}, \tilde{\mathcal{U}}, \bm{\theta}^t)] \\
        \geq & \ \frac{|\tilde{\mathcal{X}}|}{|\tilde{\mathcal{U}}|^2}  [\nabla_{\!\!\bm{\theta}^t} \overline{G}(\mathcal{X}, \bm{\theta}^t, \lambda_1)]^T 
        [\nabla_{\!\!\bm{\theta}^t} \overline{G}(\mathcal{X}, \bm{\theta}^t, \lambda_2)].
    \end{align}
    Finally, we take the expectation over $\lambda_1, \lambda_2$ and all possible batches $\mathcal{X}, \mathcal{U}$.
    Following from the convexity of $\lVert \cdot \lVert^2$, we have $\mathbb{E}(\lVert \cdot \lVert^2) \geq \lVert\mathbb{E}( \cdot )\lVert^2$.
    Therefore, we obtain
    \begin{align}
        \label{eq_37}
        \mathbb{E}_{\tilde{\mathcal{X}}, \tilde{\mathcal{U}}} [\nabla_{\!\!\bm{\theta}^t} G(\tilde{\mathcal{X}}, \bm{\theta}^t)]^T [\nabla_{\!\!\bm{\theta}^t} F(\tilde{\mathcal{X}}, \tilde{\mathcal{U}}, \bm{\theta}^t)] \geq & \frac{|\tilde{\mathcal{X}}|}{|\tilde{\mathcal{U}}|^2}  \mathbb{E}_{\mathcal{X}, \mathcal{U}} 
         \lVert\nabla_{\!\!\bm{\theta}^t}  \mathbb{E}_{\lambda} \overline{G}(\mathcal{X}, \bm{\theta}^t, \lambda) \lVert^2\\
         \geq & \frac{|\tilde{\mathcal{X}}|}{|\tilde{\mathcal{U}}|^2}   \lVert\nabla_{\!\!\bm{\theta}^t} \mathbb{E}_{\mathcal{X}, \mathcal{U}} \mathbb{E}_{\lambda} \overline{G}(\mathcal{X}, \bm{\theta}^t, \lambda) \lVert^2\\
         \label{eq_39}
        \geq & \frac{|\tilde{\mathcal{X}}|}{|\tilde{\mathcal{U}}|^2}   \lVert\nabla_{\!\!\bm{\theta}^t}  \mathbb{E}_{\tilde{\mathcal{X}}} G(\tilde{\mathcal{X}}, \bm{\theta}^t) \lVert^2,
    \end{align}
    where Inequality (\ref{eq_37}) is obtained as $\lambda_1$, $\lambda_2$ are mutually independent.
    Then we consider the updating rule of the stochastic gradient descent (SGD) algorithm:
    \begin{equation}
        \Delta \bm{\theta} = \bm{\theta}^{t+1} - \bm{\theta}^t = - \alpha^t \nabla_{\!\!\bm{\theta}^t} F(\tilde{\mathcal{X}}, \tilde{\mathcal{U}}, \bm{\theta}^t).
    \end{equation}
    Assume that the loss function on labeled data $G(\tilde{\mathcal{X}}, \bm{\theta}^t)$ is Lipschitz-smooth with the constant $L$. Following Lemma \ref{lemma}, we have
    \begin{align}
        G(\tilde{\mathcal{X}}, \bm{\theta}^{t+1}) \leq\  & G(\tilde{\mathcal{X}}, \bm{\theta}^{t}) + [\nabla_{\!\!\bm{\theta}^t} G(\tilde{\mathcal{X}}, \bm{\theta}^t)]^T \Delta \bm{\theta} + \frac{L}{2}\lVert \Delta \bm{\theta} \lVert^2 \\
        =\  & G(\tilde{\mathcal{X}}, \bm{\theta}^{t}) + \alpha^t \left[  \frac{1}{2}\alpha^t L \lVert \nabla_{\!\!\bm{\theta}^t} F(\tilde{\mathcal{X}}, \tilde{\mathcal{U}},\bm{\theta}^t)  \lVert^2 
            - [\nabla_{\!\!\bm{\theta}^t} G(\tilde{\mathcal{X}}, \bm{\theta}^t)]^T [\nabla_{\!\!\bm{\theta}^t} F(\tilde{\mathcal{X}}, \tilde{\mathcal{U}},\bm{\theta}^t)]\right].
    \end{align}
    Take the expectation over all possible $\tilde{\mathcal{X}}, \tilde{\mathcal{U}}$, and thus we obtain the following inequality using Assumption 1 and Inequality (\ref{eq_39}):
    \begin{align}
        \mathbb{E}_{\tilde{\mathcal{X}}} G(\tilde{\mathcal{X}}, \bm{\theta}^{t+1}) 
        \leq\ \mathbb{E}_{\tilde{\mathcal{X}}} G(\tilde{\mathcal{X}}, \bm{\theta}^{t}) + \alpha^t \left[ \frac{1}{2}\alpha^t \sigma L - \frac{|\tilde{\mathcal{X}}|}{|\tilde{\mathcal{U}}|^2}  \right]
        \lVert \nabla_{\!\!\bm{\theta}^t}  \mathbb{E}_{\tilde{\mathcal{X}}} G(\tilde{\mathcal{X}}, \bm{\theta}^t) \lVert^2.
    \end{align}
    As $\alpha^t \to 0$, there exists some positive constant $c$ such that for all $t$ greater than some index $\overline{t}$, we have
    \begin{align}
        \label{eq_27}
        \mathbb{E}_{\tilde{\mathcal{X}}} G(\tilde{\mathcal{X}}, \bm{\theta}^{t+1})  \leq\ \mathbb{E}_{\tilde{\mathcal{X}}} G(\tilde{\mathcal{X}}, \bm{\theta}^{t})  - \alpha^t c \lVert \nabla_{\!\!\bm{\theta} ^t}\mathbb{E}_{\tilde{\mathcal{X}}} G(\tilde{\mathcal{X}}, \bm{\theta}^t) \lVert^2,\ \  \forall t \geq \overline{t}.
    \end{align}
    We see that $\{\mathbb{E}_{\tilde{\mathcal{X}}} G(\tilde{\mathcal{X}}, \bm{\theta}^{t})\}$ is monotonically decreasing for all $t \geq \overline{t}$. As $G(\cdot)$ is computed using the cross-entropy loss over the predictions of networks, it follows $\mathbb{E}_{\tilde{\mathcal{X}}} G(\tilde{\mathcal{X}}, \bm{\theta}^{t}) \geq 0$. Therefore, $\{\mathbb{E}_{\tilde{\mathcal{X}}} G(\tilde{\mathcal{X}}, \bm{\theta}^{t})\}$ converges to a finite value. By adding Inequality (\ref{eq_27}) over all $t > \overline{t}$, we obtain
    \begin{equation}
        \label{eq_28}
        c \sum_{t=\overline{t}}^{\infty}\alpha^t \lVert \nabla_{\!\!\bm{\theta} ^t}\mathbb{E}_{\tilde{\mathcal{X}}} G(\tilde{\mathcal{X}}, \bm{\theta}^t) \lVert^2 \leq \mathbb{E}_{\tilde{\mathcal{X}}}G(\tilde{\mathcal{X}}, \bm{\theta}^{\overline{t}}) - \lim_{t \to \infty} \mathbb{E}_{\tilde{\mathcal{X}}}G(\tilde{\mathcal{X}}, \bm{\theta}^{t}) < \infty.
    \end{equation}
    It cannot exist an $\epsilon > 0$ such that $ \lVert\nabla_{\!\!\bm{\theta} ^t}\mathbb{E}_{\tilde{\mathcal{X}}} G(\tilde{\mathcal{X}}, \bm{\theta}^t)\lVert^2 > \epsilon$ for all $t$ greater than some $\hat{t}$. If so, as $ \sum_{t=0}^{\infty} \alpha^t = \infty$, the left side of Inequality (\ref{eq_28}) will come to infinity. Therefor, we must have:
    \begin{equation}
        \mathop{\lim\inf}_{t \to \infty} \lVert \nabla_{\!\!\bm{\theta} ^t}\mathbb{E}_{\tilde{\mathcal{X}}} G(\tilde{\mathcal{X}}, \bm{\theta}^t)\lVert = 0.
    \end{equation}
    In the following, we will show that $\mathop{\lim\sup}_{t \to \infty} \lVert \nabla_{\!\!\bm{\theta} ^t}\mathbb{E}_{\tilde{\mathcal{X}}} G(\tilde{\mathcal{X}}, \bm{\theta}^t)\lVert = 0$. Firstly, assume the contaary, namely
    \begin{equation}
        \label{contaary}
        \mathop{\lim\sup}_{t \to \infty} \lVert\nabla_{\!\!\bm{\theta} ^t}\mathbb{E}_{\tilde{\mathcal{X}}} G(\tilde{\mathcal{X}}, \bm{\theta}^t)\lVert \geq \epsilon > 0.
    \end{equation}
    Let $\{ m_j\}$ and $\{ n_j\}$ be the sequences of indexes such that
    \begin{equation}
        m_j < n_j <m_{j+1}, 
    \end{equation}
    \begin{equation}
        \label{eq_45}
        \frac{\epsilon}{3} <  \lVert\nabla_{\!\!\bm{\theta} ^t}\mathbb{E}_{\tilde{\mathcal{X}}} G(\tilde{\mathcal{X}}, \bm{\theta}^t) \lVert, \ \ \  m_j \leq t < n_j,
    \end{equation}
    \begin{equation}
        \lVert\nabla_{\!\!\bm{\theta} ^t}\mathbb{E}_{\tilde{\mathcal{X}}} G(\tilde{\mathcal{X}}, \bm{\theta}^t) \lVert \leq \frac{\epsilon}{3}, \ \ \  n_j \leq t < m_{j+1}.
    \end{equation}
    Since $G(\tilde{\mathcal{X}}, \bm{\theta}^t)$ is Lipschitz-smooth, it is easy to see that $ \mathbb{E}_{\tilde{\mathcal{X}}} G(\tilde{\mathcal{X}}, \bm{\theta}^t)$ is also Lipschitz-smooth. Suppose that the corresponding Lipschitz constant is $L'$.
    Let $\overline{j}$ be a sufficiently large index such that
    \begin{equation}
        \sum_{t=m_{\overline{j}}}^{\infty} \alpha^t \lVert \nabla_{\!\!\bm{\theta} ^t}\mathbb{E}_{\tilde{\mathcal{X}}} G(\tilde{\mathcal{X}}, \bm{\theta}^t)\lVert^2 < \frac{\epsilon^2}{9\sqrt{\sigma} L'}.
    \end{equation}
    For any $j \geq \overline{j}$ and any $m$ with $m_j \leq m \leq n_j - 1$, we have
    \begin{align}
        \lVert \nabla_{\!\!\bm{\theta}^{n_j}} \mathbb{E}_{\tilde{\mathcal{X}}} G(\tilde{\mathcal{X}}, \bm{\theta}^{n_j}) - \nabla_{\!\!\bm{\theta}^{m}} \mathbb{E}_{\tilde{\mathcal{X}}} G(\tilde{\mathcal{X}}, \bm{\theta}^{m}) \lVert
        \leq\  & \sum_{t=m}^{n_j-1} \lVert  \nabla_{\!\!\bm{\theta}^{t+1}} \mathbb{E}_{\tilde{\mathcal{X}}} G(\tilde{\mathcal{X}}, \bm{\theta}^{t+1}) -  \nabla_{\!\!\bm{\theta}^t} \mathbb{E}_{\tilde{\mathcal{X}}} G(\tilde{\mathcal{X}}, \bm{\theta}^t) \lVert \\
        \leq\  & L' \sum_{t=m}^{n_j-1} \lVert \bm{\theta}^{t+1} - \bm{\theta}^{t} \lVert \\
        =\  &L' \sum_{t=m}^{n_j-1} \alpha^t  \lVert \nabla_{\!\!\bm{\theta}^t} F(\tilde{\mathcal{X}}, \tilde{\mathcal{U}}, \bm{\theta}^t) \lVert.
    \end{align}
    By taking the expectation over $\tilde{\mathcal{X}},\tilde{\mathcal{U}}$, we have
    \begin{align}
        \lVert \nabla_{\!\!\bm{\theta}^{n_j}} \mathbb{E}_{\tilde{\mathcal{X}}} G(\tilde{\mathcal{X}}, \bm{\theta}^{n_j}) - \nabla_{\!\!\bm{\theta}^{m}} \mathbb{E}_{\tilde{\mathcal{X}}} G(\tilde{\mathcal{X}}, \bm{\theta}^{m}) \lVert
        \leq\  & L' \sum_{t=m}^{n_j-1} \alpha^t  \mathbb{E}_{\tilde{\mathcal{X}}, \tilde{\mathcal{U}}}\lVert \nabla_{\!\!\bm{\theta}^t} F(\tilde{\mathcal{X}}, \tilde{\mathcal{U}}, \bm{\theta}^t) \lVert \\
        \label{eq_52}
        \leq\  & \sqrt{\sigma} L' \sum_{t=m}^{n_j-1} \alpha^t  \lVert \nabla_{\!\!\bm{\theta}^t} \mathbb{E}_{\tilde{\mathcal{X}}} G(\tilde{\mathcal{X}}, \bm{\theta}^t) \lVert \\
        \label{eq_53}
        \leq\  & \frac{3\sqrt{\sigma} L'}{\epsilon}  \sum_{t=m}^{n_j-1} \alpha^t  \lVert \nabla_{\!\!\bm{\theta}^t} \mathbb{E}_{\tilde{\mathcal{X}}} G(\tilde{\mathcal{X}}, \bm{\theta}^t) \lVert^2 \\
        \leq\  & \frac{3\sqrt{\sigma} L'}{\epsilon} \frac{\epsilon^2}{9\sqrt{\sigma} L'} \\
        =\  &\frac{\epsilon}{3},
    \end{align}
    where Inequality (\ref{eq_53}) follows from Inequality (\ref{eq_45}) and Inequality (\ref{eq_52}) follows from 
    \begin{equation}
        \mathbb{E}_{\tilde{\mathcal{X}}, \tilde{\mathcal{U}}}\lVert \nabla_{\!\!\bm{\theta}^t} F(\tilde{\mathcal{X}}, \tilde{\mathcal{U}}, \bm{\theta}^t) \lVert \leq 
    \sqrt{\mathbb{E}_{\tilde{\mathcal{X}}, \tilde{\mathcal{U}}}\lVert \nabla_{\!\!\bm{\theta}^t} F(\tilde{\mathcal{X}}, \tilde{\mathcal{U}}, \bm{\theta}^t) \lVert^2 
    } \leq
    \sqrt{\sigma} \lVert \nabla_{\!\!\bm{\theta}^t} \mathbb{E}_{\tilde{\mathcal{X}}} G(\tilde{\mathcal{X}}, \bm{\theta}^t) \lVert.
    \end{equation} 
    Thus, we have
    \begin{equation}
        \lVert   \nabla_{\!\!\bm{\theta}^{m}} \mathbb{E}_{\tilde{\mathcal{X}}} G(\tilde{\mathcal{X}},\bm{\theta}^{m}) \lVert \leq  \lVert \nabla_{\!\!\bm{\theta}^{n_j}} \mathbb{E}_{\tilde{\mathcal{X}}} G(\tilde{\mathcal{X}},\bm{\theta}^{n_j}) \lVert + \frac{\epsilon}{3}
        \leq \frac{2\epsilon}{3},\ \ \  \forall j \geq \overline{j},m_j \leq m \leq n_j -1.
    \end{equation}
    As the inequality holds for all $m \geq m_{\overline{j}}$, we finally obtain
    \begin{equation}
        \lVert   \nabla_{\!\!\bm{\theta}^{m}} \mathbb{E}_{\tilde{\mathcal{X}}} G(\tilde{\mathcal{X}},\bm{\theta}^{m}) \lVert \leq \frac{2\epsilon}{3}, \ \ \  \forall m \geq m_{\overline{j}},
    \end{equation}
    which contradicts Inequality (\ref{contaary}), implying that 
    \begin{equation}
        \mathop{\lim \inf}_{t \to \infty} \lVert \nabla_{\!\!\bm{\theta}^t} \mathbb{E}_{\tilde{\mathcal{X}}} G(\tilde{\mathcal{X}},\bm{\theta}^t) \lVert 
        = \mathop{\lim \sup}_{t \to \infty} \lVert \nabla_{\!\!\bm{\theta}^t} \mathbb{E}_{\tilde{\mathcal{X}}} G(\tilde{\mathcal{X}},\bm{\theta}^t) \lVert
        = 0.
    \end{equation}
    Therefore, we prove that $\lim_{t \to \infty} \lVert \nabla_{\!\!\bm{\theta}^t} \mathbb{E}_{\tilde{\mathcal{X}}} G(\tilde{\mathcal{X}},\bm{\theta}^t) \lVert = 0$.
\end{proof}

\section{Details of Experiments}
\textbf{Datasets. }(1) The CIFAR-10 / CIFAR-100 datasets consist of 60,000 32x32 colored images of 10 / 100 classes, 50,000 for training and 10,000 for test. Following the common practice of SSL \cite{oliver2018realistic, berthelot2019mixmatch, verma2019interpolation, tarvainen2017mean, athiwaratkun2018there}, we hold out 5k images from the training set as the validation set. Images are normalized with channel means and standard deviations for pre-processing. Then data augmentation is performed by 4x4 random translation followed by random horizontal flip \cite{He_2016_CVPR, huang2019convolutional}. On CIFAR-10, we preserve 100, 200 and 400 labels per class respectively, corresponding to 1000, 2000, 4000 labeled samples in total. All other samples are unlabeled. We randomly split the dataset for 5 times to conduct multiple experiments, and report the mean test errors associated with standard deviations. Similarly, On CIFAR-100, evaluation is performed with 40 and 100 randomly preserved labeled samples per class. (2) SVHN consists of 32x32 colored images of digits. 73,257 images for training, 26,032 images for testing and 531,131 images for additional training are provided. Following \cite{luo2018smooth, tarvainen2017mean}, we merely perform random 2x2 translation to augment the training set, and hold out 1,000 images for validation. Similar to CIFAR, we randomly preserve 500 and 1,000 labels for experiments. (3) STL-10 \cite{coates2011analysis} contains 5,000 training examples divided into 10 predefined folds with 1000 examples each, and 100,000 unlabeled images drawn from a similar—but not identical—data distribution. All the samples are 96x96 colored images. We use the same experimental protocol as \cite{berthelot2019mixmatch}.

\textbf{Networks.}
Our experiments are based on a 13-layer CNN (CNN-13) and the Wide-RestNet-28-2 (WRN-28) network. The CNN-13 network has been adopted as the standard model for experiments by state-of-the-art SSL algorithms \cite{verma2019interpolation, tarvainen2017mean, athiwaratkun2018there, miyato2018virtual, luo2018smooth, park2018adversarial}. Following \cite{verma2019interpolation}, we remove the Gaussian noise layer and the dropout layer in the network. Other methods use these techniques if mentioned in their original papers, which provide stronger regularization. Some recent works adopt the WRN-28 network \cite{oliver2018realistic, berthelot2019mixmatch} in their experiments. We also implement \textit{Meta-Semi} with WRN-28 to present comparisons with them.

\textbf{Large Validation Set.} We note that the validation set we use may be relatively large in some settings (e.g. 5,000 for validation on CIFAR-10 with 1,000 labeled examples). However, since most prior SSL methods do so, we simply follow them to produce comparable results with them in the paper. On the other hand, as discussed in the sensitivity test, our method is less sensitive to the only tunable hyper-parameter $\beta$ and thus requires less validation efforts. To further demonstrate this point, we perform a four-fold cross-validation on CNN-13 based \textit{Meta-Semi} with 1,000 labeled samples on CIFAR-10 to search for the optimal $\beta$. Our method achieves a test error of $10.96 \pm 0.56\%$, which is slightly higher than the $10.27 \pm 0.66 \%$ of using addition 5,000 labeled samples for validation, but still significantly outperforms baselines.

\textbf{Training details.}
The CNN-13 network uses the SGD optimizer with a Nesterov momentum of 0.9. The L2 regularization coefficient is set to 1e-4, and the initial learning rate is set to 0.1. For all experiments with CNN-13, we train the network for 600 epochs using the cosine learning rate annealing technique \cite{loshchilov2016sgdr, huang2017snapshot, verma2019interpolation}. The batch size of labeled samples and unlabeled samples are set to 25 and 75 respectively. To generate pseudo labels for unlabeled samples, we use an exponential moving average on model parameters with a decay rate of 0.999. For WRN-28, we adopt exactly the same training details as \cite{berthelot2019mixmatch} except for the batch size: we use 32 for labeled samples and 96 for unlabeled samples. The ratio of labeled/unlabeled samples in each mini-batch is always set to 1:3 in \textit{Meta-Semi}, which consistently achieves excellent performance on the validation set, and does not need to be tuned for the specific SSL task.

\textbf{Baselines.}
Our method is compared with several state-of-the-art baselines including SSL algorithms and a meta-reweighting method.
\begin{itemize}
    \item $\Pi$-model \cite{laine2016temporal} enforces the model predictions to remain the same when different augmentation and dropout modes are performed.
    \item Temp-ensemble \cite{laine2016temporal} attaches a soft pseudo label for each unlabeled sample by performing a moving average on the historical predictions of networks.
    \item Mean Teacher (MT) \cite{tarvainen2017mean} establishes a teacher network by performing exponential moving average on the parameters of the model, and leverages the teacher networks to produces supervision for unlabeled data.
    \item Virtual Adversarial Training (VAT) \cite{miyato2018virtual} adds adversarial perturbations to the samples and enforce the model to have the same predictions on perturbed samples and the original samples.
    \item Smooth Neighbors on Teacher Graphs (SNTG)\cite{luo2018smooth} constructs a teacher graph to regularize the feature distribution of unlabeled samples.
    \item Learning to Reweight \cite{ren2018learning} proposes to reweight different training samples by solving a similar meta-learning problem to us. Since their original algorithm requires labels of all the training, we adopt a version modified for SSL in this paper. In specific, we retain our approach of generating pseudo-labeled samples, but use their reweighting strategy.
    \item MT + Fast SWA \cite{athiwaratkun2018there} is an improved MT algorithm using a fast stochastic weight averaging optimizer.
    \item Interpolation Consistency Training (ICT) \cite{verma2019interpolation} encourages the prediction on an interpolation of unlabeled
    samples to be consistent with the interpolation of the predictions on those points. They first use MixUp augmentation in deep SSL.
    \item MixMatch \cite{berthelot2019mixmatch} is a holistic deep SSL approach that integrates various dominant consistency regularization techniques.
\end{itemize}
We implement these methods in the same codebase, and search for the best hyper-parameters for them on the validation set according to the recommendations provided by their original papers. Notably, for MixMatch \cite{berthelot2019mixmatch}, we fix the sharpening temperature $T=0.5$ and the number of unlabeled augmentations $K=2$, and adjust the $\alpha$  parameter for Beta distribution and the unsupervised loss coefficient $\lambda_{\tilde{\mathcal{U}}}$, as suggested by the paper. We first reproduce the CIFAR-10 results of MixMatch reported by their paper, and then tune $\alpha$ and $\lambda_{\tilde{\mathcal{U}}}$ on the validation set of CIFAR-100.

\end{document}